\documentclass[10pt]{article}

\usepackage{arxiv_styles}

\title{A General $\widetilde{\Omega}(\sqrt{T \gamma_T})$ Lower Bound for Kernel Bandits}

\author{
Chenkai Ma \\
National University of Singapore \\
\texttt{chenkai.ma@u.nus.edu}
\and
Jonathan Scarlett \\
National University of Singapore \\
\texttt{scarlett@comp.nus.edu.sg}
}

\date{}

\begin{document}

\maketitle

\begin{abstract}
    \noindent
    The kernel bandit problem consists of sequentially optimizing an unknown function with noisy feedback, where the function has bounded norm in a given Reproducing Kernel Hilbert Space (RKHS). 
    A central quantity in the regret analysis of kernel bandits is the maximum information gain $\gamma_T$.  In particular, the best existing upper bounds scale as $\sqrt{T\gamma_T}$ up to log factors, and nearly-matching lower bounds have been derived for specific kernels such as squared exponential and Mat\'ern.  However, lower bounds for general kernels are lacking, thus making it unclear in what generality the upper bounds are near-optimal.  In this paper, we establish a general $\Omega(\sqrt{T\gamma_T/\log T})$ minimax regret lower bound for non-constant continuous kernels on compact domains, establishing near-optimality (within log factors) in a very general sense.  We show that the log factor appearing in this bound is unavoidable in general, but that it can be removed under certain conditions.  Among other things, our findings imply that the minimax-optimal scaling is exactly $\Theta(\sqrt{T\gamma_T})$ (i.e., within constant factors) for the Mat\'ern-$\nu$ kernel with $\nu \in (0,2)$, $\gamma$-exponential kernel with $\gamma \in (0,2)$, and certain piecewise-polynomial kernels. 
\end{abstract}

\section{Introduction}
\label{sec:introduction}

The kernel bandit problem consists of sequentially optimizing an unknown function with noisy observations \citep{Srinivas10}. 
We consider the frequentist setting, in which the function has bounded norm in a Reproducing Kernel Hilbert Space (RKHS).
At each round, a learner queries the function at some point, and observes its noisy value. The objective is to minimize cumulative regret compared to the best action.

A central quantity in the regret analysis of kernel bandits is the maximum information gain $\gamma_T$, which measures the largest information gain from $T$ observations under the surrogate Gaussian process (GP) model \citep{Srinivas10,Chowdhury17}. 
The best existing upper bounds for general kernels scale as $\sqrt{T\gamma_T}$ up to logarithmic factors \citep{Valko13,Camilleri21,Salgia21}, and nearly-matching lower bounds have been established for the squared exponential and Mat\'ern kernels \citep{Scarlett17,Cai21}, and extended to other isotropic kernels \citep{Lee25}.  In this paper, we are interested in the following question: \emph{Does $\sqrt{T\gamma_T}$ characterize the minimax regret for general kernels beyond the specific ones that have near-matching lower bounds?}

To address this question, it turns out to be helpful to connect the maximum information gain to the effective dimension, which measures the effective number of feature directions above a certain threshold \citep{Alaoui15,Derezinski20}.  This quantity has frequently appeared in the study of \emph{upper bounds}.  In particular, \citep{Camilleri21} presents the RIPS algorithm based on regularized experimental design, whose regret is related to both the maximum information gain and the effective dimension. More recently, \citep{Janz26} provides sharp estimates for both quantities across broad kernel classes.

Our main contributions are outlined as follows:
\begin{itemize}
    \item We prove an $\Omega(\sqrt{Td_T})$ lower bound on finite domains, where $d_T$ is the effective dimension of a regularized $D$-optimal design, and we extend this to compact domains by a finite approximation. We show that this implies $\Omega(\sqrt{T\gamma_T/\log T})$ regret for every non-constant continuous positive semidefinite kernel on a compact subset of $\mathbb{R}^d$.
    \item We show that the logarithmic factor is unavoidable in general, since on small finite domains, $\gamma_T = \Theta(\log T)$ and the minimax regret is $\Theta(\sqrt{T})$.  Moreover, on $[0,1]^d$, we derive a sufficient condition that leads to $\Omega(\sqrt{T\gamma_T})$ regret, which is satisfied if $\gamma_T = \Theta(T^{\alpha})$ for some $\alpha > 0$, and holds for the Mat\'ern kernel among others.
    \item For comparison, we adapt an upper bound from \citep{Camilleri21} to attain scaling $O(\sqrt{T\gamma_T\log T})$ on $[0,1]^d$ under a suitable H\"{o}lder condition, and sharpen it to $O(\sqrt{T\gamma_T})$ in some kernel settings. The resulting bounds match at $\Theta(\sqrt{T\gamma_T})$ for the Mat\'ern-$\nu$ kernel with $\nu \in (0,2)$, $\gamma$-exponential kernel with $\gamma \in (0,2)$, and $q$-piecewise polynomial kernel with $q = 0, d\geq 3$ or $q = 1$.  For other kernels or general settings that we consider, a gap of $\sqrt{\log T}$ or $\log T$ remains.
\end{itemize}

\subsection{Setup}
\label{sec:setup_finite}

Let $\mathcal{X} \subseteq \mathbb{R}^d$ be nonempty and finite, and let $k: \mathcal{X}\times \mathcal{X} \to \mathbb{R}$ be a non-constant continuous positive-semidefinite kernel satisfying $\sup_{x\in \mathcal{X}}  k(x,x) \leq 1$.
Let $\mathcal{H}_k$ be the RKHS induced by $k$, and write $\phi(x) := k(x,\cdot) \in \mathcal{H}_k$.
Define $\Delta_k := \max_{x,x'\in \mathcal{X}} \norm{\phi(x) - \phi(x')}_{k} $, which is positive since $k$ is non-constant.
For $B > 0$, define $\mathcal{F}_k(B) := \{g \in \mathcal{H}_k: \norm{g}_{k} \leq B \} $. We are interested in optimizing a function $f \in \mathcal{F}_k(B)$.

Let $T\in \mathbb{Z}_{+}$ be the time horizon. At time $t \in [T]$, a possibly randomized adaptive (non-anticipatory) policy $\pi$ chooses $x_t \in \mathcal{X}$ and observes $y_t = f(x_t) + \xi_t$, where $(\xi_t)_{t\geq 1}$ is sampled i.i.d. (independent of the randomness of the policy) from $N(0,\sigma^2)$ for some $\sigma > 0$.
For this policy $\pi$, define the \emph{cumulative regret}
\begin{align}
    R_T(\pi, f) := \sum_{t=1}^T \bigl(\max_{x\in \mathcal{X}} f(x) - f(x_t)\bigr). \label{eq:regret_pi_f}
\end{align}
The \emph{minimax regret} is then defined as
\begin{align}
    R_T^{*}
    := \inf_{\pi} \sup_{f\in\mathcal{F}_k(B)} \mathbb{E}[R_T(\pi,f)]. \label{eq:minimax_regret}
\end{align}

\paragraph{Effective Dimension.}

Our initial lower bound will be stated in terms of the effective dimension \citep{Camilleri21}, which is related to regularized optimal experimental design \citep{Derezinski20, Allen21}.
Let $E := \mathrm{span} \{\phi(x): x \in \mathcal{X}  \} $, which is a finite-dimensional subspace of $\mathcal{H}_k$ since $\mathcal{X}$ is finite. Let $n = \dim (E)$, and $e_1,\ldots, e_n$ be an orthonormal basis of $E$. We write $I_E$, $\det_E$, and $\tr_E$ for the identity map, determinant, and trace on $E$, respectively. 
For example, 
$\det_E (A) = \det([\langle e_i, Ae_j \rangle_k ]_{i,j=1}^n)$, where the right-hand side is an ordinary determinant of an $n\times n$ matrix.  (If $A$ acts entirely within $E$, then $\det_E(A) = \det(A)$ and $\tr_E(A) = \tr(A)$.)

Let $\Delta(\mathcal{X})$ denote the set of probability distributions on $\mathcal{X}$.
For $\nu \in \Delta(\mathcal{X)}$, define the uncentered second moment operator on $E$ by
\begin{align}
    M_{\nu}
    := \mathbb{E}_{X\sim \nu}[\phi(X)\otimes\phi(X)]
    = \sum_{x\in\mathcal{X}} \nu(x) \phi(x)\otimes \phi(x), \label{eq:moment_operator}
\end{align}
where $u\otimes v$ is the operator on $E$ given by $(u\otimes v)w = u \langle v,w \rangle_{k} $. Note that for the linear kernel, $M_{\nu} = \sum_{x\in \mathcal{X}}\nu(x) xx^{T} $, the uncentered second moment matrix.
Define $\rho:= \sigma^2/(B^2T) > 0$ to be a regularization level, and
consider a regularized $D$-optimal design
\begin{align}
    \nu_{T}
    \in \argmax_{\nu \in \Delta(\mathcal{X})} \log \det_E \bigl(M_{\nu} + \rho I_{E}\bigr), \label{eq:d_optimal_design}
\end{align}
where a maximizer exists because $\Delta(\mathcal{X})$ is compact and the objective is continuous.
We will show that every $\nu_T$ satisfying \cref{eq:d_optimal_design} gives the same operator in \cref{eq:moment_operator} (see \Cref{lem:unique_operator}), so we define $M_{T} := M_{\nu_{T}}$.
Then, the effective dimension (at regularization level $\rho$) is defined by\footnote{\citep[Lemma 3]{Camilleri21} does not use the subscript $E$ for the definition of the effective dimension because it adopts finite-dimensional notation.}
\begin{align}
    d_T
    := \tr_E\bigl(M_{T}(M_T + \rho I_E)^{-1}\bigr). \label{eq:effective_dimension}
\end{align}
Let $\lambda_{1},\ldots, \lambda_n $ be the eigenvalues of $M_{T}$, all of which are nonnegative (see \Cref{lem:positive_definite}). Then $d_T = \sum_{j=1}^n \frac{\lambda_j}{\lambda_j + \rho} $, and thus an eigendirection contributes roughly one when its eigenvalue is much larger than $\rho$, and contributes little when its eigenvalue is much smaller than $\rho$.

\paragraph{Maximum Information Gain.}
Define $\lambda := \sigma^2 / B^2$, which is the noise variance when the function and noise are normalized by $B$.
To compare with standard kernel bandits literature, we will also provide lower bounds stated in terms of the maximum information gain \citep{Srinivas10,Chowdhury17}, defined by
\begin{align}
    \gamma_T
    & :=  \frac{1}{2} \max_{x_{1:T} \in \mathcal{X}^T} \log\det (I_T + \lambda^{-1} K_T(x_{1:T}) ), \label{eq:gamma_T}
\end{align}
where $K_T(x_{1:T}) = [k(x_s,x_{s'})]_{s,s'\in[T]}$.

\paragraph{Other Notations.}
For self-adjoint operators $A,B$ on the same inner product space $V$, we write $A \preceq B$ to denote that $B-A$ is positive semidefinite, i.e., $ \langle Av, v \rangle \leq \langle Bv, v \rangle $ for every $v \in V$. We write $B \succeq A$ to denote $A \preceq B$. Similarly, we use $\succ, \prec$ to indicate positive definiteness. 
Assuming $k, B, \sigma$ are fixed, we use $O(\cdot)$ and $\Omega(\cdot)$ for asymptotic upper and lower bounds as $T \to \infty$, $\widetilde{O}(\cdot), \widetilde{\Omega}(\cdot)$ to additionally hide polylogarithmic factors, and $\Theta(\cdot)$ to mean both $O(\cdot)$ and $\Omega(\cdot)$.

\subsection{Related Work}

\paragraph{Upper Bounds for Kernel Bandits.}
\citep{Srinivas10} presents the GP-UCB algorithm, which connects the cumulative regret to $\gamma_T$. In the frequentist setting, its analysis gives $\widetilde{O}(\gamma_T\sqrt{T})$ regret. Later, \citep{Chowdhury17} sharpens the analysis but ultimately attains the same $\widetilde{O}(\gamma_T\sqrt{T})$ behavior. \citep{Valko13} is the first to attain $\widetilde{O}(\sqrt{T\gamma_T})$ regret, using the SupKernelUCB algorithm. Subsequent approaches yielding $\widetilde{O}(\sqrt{T\gamma_T})$ include experimental design \citep{Camilleri21}, domain shrinking \citep{Salgia21}, and batched elimination \citep{Li22}.  The latter shows that $O(\log\log T)$ batches are sufficient, and \citep{Ma26} refines this by improving the batch count and hidden $\widetilde{O}(\cdot)$ terms.

Most of these results are stated on finite domains, and can be extended to $[0,1]^d$ through suitable finite approximations.
Although our paper mainly focuses on the lower bounds, for comparison we establish an $O(\sqrt{T\gamma_T\log T})$ upper bound based on \citep{Camilleri21}, and sharpen it to $O(\sqrt{T\gamma_T})$ in several kernel settings by further combining results from \citep{Lee25,Janz26}. We also combine domain discretization with MOSS (\cite{Audibert09,Lattimore20}, Theorem 9.1) to obtain $O(\sqrt{T\gamma_T})$ in some other kernel settings.
We are not aware of any sharper upper bounds from earlier works.

\paragraph{Lower Bounds for Kernel Bandits.}

\citep{Scarlett17} establishes an $\Omega(\sqrt{T (\log T)^{d/2}})$ lower bound for the squared exponential (SE) kernel, and $\Omega(T^{(\nu + d)/ (2\nu + d)})$ for the Mat\'ern kernel by constructing a set of hard-to-distinguish bump functions. \citep{Cai21} uses an alternative change-of-measure proof to improve these bounds in terms of their dependence on the error probability. These lower bounds are kernel-specific, and derived directly without using information gain. 

More recently, \citep{Iwazaki26} obtains a stronger $\Omega(\sqrt{T (\log T/ \log\log T)^{d}})$ bound for the SE kernel on the $d$-dimensional sphere $\mathbb{S}^d$. 
We will obtain the same scaling for the hypercube $[0,1]^d$ by combining our general $\Omega(\sqrt{T\gamma_T/\log T})$ bound with the rate of $\gamma_T$ given by \citep{Janz26}.  Thus, we will improve the SE kernel lower bound on $[0,1]^d$ compared to \cite{Scarlett17}, whereas our Mat\'ern lower bound will turn out to be the same as theirs up to constant factors (see \Cref{tab:gamma_T_rates} below).

To our knowledge, sharp lower bounds have not been established for the rational quadratic kernel and one-dimensional Dirichlet kernel (among many others).

\paragraph{Maximum Information Gain.}

The maximum information gain $\gamma_T$ has been studied extensively.
\citep{Srinivas10} gives $\gamma_T =  \widetilde{O}(T^{(d(d+1))/(2\nu+d(d+1))})$ for the Mat\'ern-$\nu$ kernel and $\gamma_T = O((\log T)^{d+1})$ for the squared exponential kernel. \citep{Vakili21b} improves the former to $\widetilde{O}(T^{d/(2\nu + d)})$ based on the Mercer eigenvalue decay under the assumption of uniformly bounded Mercer eigenfunctions. 
\citep{Iwazaki26} develops a spherical-harmonic approach to obtain $O((\log T)^{d+1})/(\log\log T)^d$ for the squared exponential kernel on compact subsets of $\mathbb{R}^d$. 
The results from \citep{Scarlett17,Valko13} can also be combined to infer that $\gamma_T$ has scaling at least $T^{d/(2\nu +d)}$ (Mat\'ern) or $(\log T)^d$ (SE) up to dimension-independent log factors. 
Recently, \citep{Janz26} shows that $\gamma_T = \Theta(T^{d/(2\nu+d)})$ for the Mat\'ern-$\nu$ kernel, and $\gamma_T = \Theta((\log T)^{d+1}/(\log\log T)^d)$ for the squared exponential kernel.  The same work also connects $\gamma_T$ to an empirical effective dimension, which is different from our definition in \cref{eq:effective_dimension} allowing fractional design. 
The bounds on $\gamma_T$ alone do not imply minimax regret lower bounds, but are helpful when we compare our lower bounds with kernel-specific upper bounds. 

\paragraph{Smoothness and Function-Space Approaches.}

Kernel bandits are closely related to continuum-armed bandits over smooth function classes. 
\citep{Singh21} derives minimax lower bounds for Besov classes, while \citep{Liu21} develops local polynomial algorithms attaining $\widetilde{O}(T^{(\alpha + d)/(2\alpha + d)})$ regret for H\"{o}lder smoothness ($\alpha > 1$). 
\citep{Lee25} studies six kernel families through spectral decay and links to H\"{o}lder and Besov smoothness. Its findings include kernel-specific regret lower bounds, such as $\widetilde{\Omega}(T^{(\gamma+2d)/(2\gamma + 2d)})$ for the $\gamma$-exponential kernels and $\tilde{\Omega}(T^{(2q+1+2d)/(4q+2+2d)})$ for the piecewise-polynomial kernels (for admissible kernel parameters). 
Our lower bounds directly use the effective dimension without requiring the identification of an RKHS with H\"{o}lder or Besov spaces.

\paragraph{Experimental Design and Effective Dimension.}

\citep{Camilleri21} analyzes kernel bandits using the RIPS algorithm and regularized experimental design.  Theorem 3 therein bounds regret through the largest $G$-optimal design value over surviving finite action sets, and its Lemmas 3 and 15 relate this value to the effective dimension and regularized $D$-optimal design.
We similarly use $D$-optimal design to define the effective dimension, and then use it to construct hard functions needed in the analysis of the lower bound. We also extend the upper bound in \citep[Theorem 3]{Camilleri21} from finite domains to $[0,1]^d$, establishing an $O(\sqrt{T\gamma_T\log T})$ regret bound that holds under a H\"{o}lder condition.

\section{The Finite Domain Setting}
\label{sec:finite}

\subsection{Statement of Results}

We state the main results for finite domains as follows, and defer infinite domains to Section \ref{sec:compact_X}.

\begin{theorem}
    \label{thm:lower_bound_finite_X_d_T}
    Consider the setup in \Cref{sec:setup_finite} where $\mathcal{X}$ is finite.
    Then,
    \begin{align*}
        R^*_T
        \geq \min \bigl\{ 1, \Delta_k \bigr\} \cdot \frac{\sigma}{16} \sqrt{T d_T},
    \end{align*}
    where $\Delta_k := \max_{x,x'\in \mathcal{X}} \norm{\phi(x) - \phi(x')}_{k} $.
\end{theorem}

\begin{corollary}
    \label{cor:lower_bound_finite_X_gamma_T}
    Consider the setup in \Cref{sec:setup_finite} where $\mathcal{X}$ is finite.
    Then,
    \begin{align*}
        R_T^{*} 
        \geq \min \bigl\{ 1, \Delta_k \bigr\} \cdot \frac{\sqrt{2}\sigma}{16} \sqrt{\frac{T\gamma_T}{1+\log(1+B^2T/\sigma^2)}},
    \end{align*}
    where $\Delta_k := \max_{x,x'\in \mathcal{X}} \norm{\phi(x) - \phi(x')}_{k} $.
\end{corollary}

\subsection{Proofs}
\label{sec:proof_finite}

The main idea is to construct hard-to-distinguish functions in $\mathcal{F}_k(B)$ using $d_T$, the policy-independent effective number of feature directions at scale $\rho$.
\Cref{lem:facts} shows that the regularized $D$-optimal design $\nu_T$ provides normalized feature vectors $u_x$ (defined in \cref{eq:ux} below) that have bounded norms and a bounded second moment operator (see \cref{eq:fact_1,eq:fact_2,eq:fact_3,eq:fact_4}). 
When $d_T \geq 4$ (\Cref{lem:regret_high_ds}), we use these vectors to construct functions in $\mathcal{F}_k(B)$, each maximized at its corresponding support point with value of order $\sigma\sqrt{d_T/T}$. 
We then consider a zero-function model where the learner's actions are independent of the randomly chosen hard function. 
The bounds in \Cref{lem:facts} are used to lower bound regret in the zero-function model, and upper bound the KL divergence between this model and the hard-function model. Then we apply a change-of-measure argument to transfer the regret lower bound to the hard-function model.
When $d_T < 4$ (\Cref{lem:regret_2}), a separate construction using two opposite functions based on maximally separated feature vectors gives the complementary lower bound. Combining the two cases proves \Cref{thm:lower_bound_finite_X_d_T}, and relating $d_T$ to $\gamma_T$ (\Cref{lem:d_T_bounds}) yields \Cref{cor:lower_bound_finite_X_gamma_T}.

We proceed to state and prove \Cref{lem:regret_high_ds}, which is the core part of the argument.  The other claims above, along with their proofs, are given in \Cref{app:proof_finite}.
We define the regularized operator $V_T := M_T  + \rho I_E $, and noting that $V_T \succ 0$ (see \Cref{lem:positive_definite}) and $d_T > 0$ (see \Cref{lem:d_T_is_positive}),
we define the following for every $x \in \mathcal{X}$:
\begin{align}
    u_x:= d_T^{-1/2} V_T^{-1/2} \phi(x) \in \mathcal{H}_k. \label{eq:ux}
\end{align}
These are normalized feature vectors, and their key properties are established in \eqref{eq:fact_1}--\eqref{eq:fact_4} (\Cref{lem:facts}).

\begin{lemma}
    \label{lem:regret_high_ds}
    Consider the setup in \Cref{sec:setup_finite}.
    If $d_T \geq 4$,
    then
    \begin{align}
        R_T^{*}
        \geq \frac{\sigma}{16} \sqrt{T d_T}.
    \end{align}
\end{lemma}

\begin{proof}    
    \textbf{Step 1: Construct hard functions.}
    For every $z \in \mathrm{supp}(\nu_T)$, define
    \begin{align}
        f_z := \frac{B}{4}\sqrt{\frac{\rho}{d_T}} V_T^{-1} \phi(z) \in \mathcal{H}_k.
    \end{align}
    We first show that $f_z \in \mathcal{F}_k(B)$.
    \Cref{lem:positive_definite} shows that
    $V_T^{-2} \preceq \rho^{-1}V_T^{-1}  $ and $V_T^{-1}, V_T^{-1/2}$ are self-adjoint.
    Then
    \begin{align}
        \norm{f_z}_k^2
        & = \frac{B^2\rho}{16 d_T} \langle \phi(z), V_T^{-2}\phi(z) \rangle_k \\
        & \leq \frac{B^2}{16 d_T} \langle \phi(z), V_T^{-1}\phi(z) \rangle_k \\
        & = \frac{B^2}{16} \label{eq:use_fact_2} \\
        & \leq B^2,
    \end{align}
    where \cref{eq:use_fact_2} follows from \cref{eq:fact_2,eq:ux}. We have thus shown $f_z \in \mathcal{F}_k(B)$. 
    For any $x \in \mathcal{X}$,
    the reproducing property and definition of $u_x$ give that
    \begin{align}
        f_z(x) 
        = \frac{B}{4}\sqrt{\rho d_T} \langle u_z, u_x \rangle_k, \label{eq:f_z_x}
    \end{align}
    which combined with \cref{eq:fact_1,eq:fact_2} gives
    \begin{align}
        f_z(z) = B\sqrt{\rho d_T}/4,
        \quad
        \abs{f_z(x)} \leq B\sqrt{\rho d_T}/4.
        \label{eq:max_f_z}
    \end{align}

    \textbf{Step 2: Change of measure.} 
    We will apply a well-known change of measure result, stated in \Cref{lem:change_of_measure} in \Cref{app:technical_lemmas}.
    We first bound the relevant KL divergence.
    Fix any policy $\pi$, and let
    $H_{t} := (X_1,Y_1,\ldots, X_{t},Y_{t}) $ be the history until time $t$ (the uppercase $X_t$ and $Y_t$ are random variables, whereas lowercase $x_t$ and $y_t$ are realizations). Let $P$ be the joint law on $(Z,H_T)$ obtained by first drawing $Z \sim \nu_T$, then running $\pi$ against $f_Z$. Let $Q$ be the joint law on $(Z,H_T)$ when the policy is instead run on the zero function, and $Z \sim \nu_T$ is drawn independently from the interaction.
    By the chain rule for KL divergence, 
    \begin{align}
        & D_{\mathrm{KL}} (Q \Vert P) \notag \\
        &= D_{\mathrm{KL}}(Q_Z \Vert P_Z) + \sum_{t=1}^T \mathbb{E}_Q[D_{\mathrm{KL}}(Q_{X_t\vert Z, H_{t-1}}  \Vert P_{X_t\vert Z, H_{t-1}} )] \notag \\
        &\quad +  \sum_{t=1}^T \mathbb{E}_Q[D_{\mathrm{KL}}(Q_{Y_t\vert Z, H_{t-1}, X_t}  \Vert P_{Y_t\vert Z, H_{t-1}, X_t}) ] \label{eq:kl_decompose} \\
        &= \sum_{t=1}^T \mathbb{E}_Q[D_{\mathrm{KL}}(Q_{Y_t\vert Z, H_{t-1}, X_t} \Vert P_{Y_t\vert Z, H_{t-1}, X_t} ) ] \label{eq:same_policy} \\
        &= \frac{1}{2\sigma^2} \sum_{t=1}^T \mathbb{E}_Q [f_Z(X_t)^2], \label{eq:KL_gaussian}
    \end{align}
    where \cref{eq:same_policy} follows because $P_Z = Q_Z$ and the same policy is used for both $f_Z$ and the zero function, and \cref{eq:KL_gaussian} follows from the KL divergence formula between two Gaussians.
    Under $Q$, the random variables $Z$ and $X_t$ are independent,
    which implies  $\mathbb{E}_Q [ f_Z(X_t)^2] = \mathbb{E}_Q[\mathbb{E}_Q[f_Z(X_t)^2 \vert X_t] ] = \mathbb{E}_Q [\mathbb{E}_{Z}[f_Z(X_t)^2]]  $. 
    Observe that, for any fixed $x \in \mathcal{X}$,
    \begin{align}
        \mathbb{E}_Z[f_Z(x)^2]
        & = \frac{B^2\rho d_T}{16} \mathbb{E}_Z[\abs{\langle u_Z, u_x \rangle_k}^2] \label{eq:use_f_z_x} \\
        & = \frac{B^2\rho d_T}{16} \langle u_x, \mathbb{E}_Z[u_Z\otimes u_Z]u_x \rangle_k \label{eq:use_otimes} \\
        & \leq \frac{B^2\rho}{16} \langle u_x, u_x \rangle_k \label{eq:use_fact_3} \\
        & \leq \frac{B^2\rho}{16}, \label{eq:use_fact_1}
    \end{align}
    where \cref{eq:use_f_z_x} follows from \cref{eq:f_z_x}, \cref{eq:use_otimes} follows from the definition of $\otimes$ and linearity of expectation, \cref{eq:use_fact_3} follows from \cref{eq:fact_3}, and \cref{eq:use_fact_1} follows from \cref{eq:fact_1}.
    Combining with \cref{eq:KL_gaussian} gives
    \begin{align}
        D_{\mathrm{KL}} (Q \Vert P)
        \leq \frac{TB^2\rho}{32\sigma^2}
        = \frac{1}{32}. \label{eq:kl_bound}
    \end{align}
    We now proceed to bound $\mathbb{E}_Q[R_T(\pi, f_Z)]$.
    For any $x \in \mathcal{X}$, we have
    \begin{align}
        \mathbb{E}_Z[f_Z(x)] =
        & \frac{B\sqrt{\rho d_T}}{4} \langle \mathbb{E}_Z[u_Z], u_x \rangle_k \label{eq:use_f_z_x_2} \\
        & \leq \frac{B\sqrt{\rho}}{4} \norm{u_x}_k \label{eq:use_fact_4} \\
        & \leq \frac{B\sqrt{\rho}}{4}, \label{eq_f_z_xt}
    \end{align}
    where \cref{eq:use_f_z_x_2} follows from \cref{eq:f_z_x}, \cref{eq:use_fact_4} follows from \cref{eq:fact_4} and Cauchy-Schwarz, and \cref{eq_f_z_xt} follows from \cref{eq:fact_1}.
    Recalling $\mathbb{E}_Q [ f_Z(X_t)] = \mathbb{E}_Q [\mathbb{E}_{Z}[f_Z(X_t)]]$,  \cref{eq:max_f_z,eq_f_z_xt} then give
    \begin{align}
        \mathbb{E}_Q[R_T(\pi, f_Z)] 
        & = \sum_{t=1}^T \mathbb{E}_Q \Bigl[\max_{x\in\mathcal{X}}f_Z(x) - f_Z(X_t) \Bigr] \\
        & \geq BT \biggl( \frac{\sqrt{\rho d_T}}{4} - \frac{\sqrt{\rho}}{4} \biggr) \geq \frac{TB\sqrt{\rho d_T}}{8}, \label{eq:regret_under_q}
    \end{align}
    where the last step uses $d_T \geq 4$.
    Finally, \cref{eq:max_f_z} implies $0 \leq R_T(\pi, f_Z) \leq BT\sqrt{\rho d_T} / 2$.
    Applying the standard change of measure result from \Cref{lem:change_of_measure} in \Cref{app:technical_lemmas} (total variation + Pinsker's inequality) with $H = R_T(\pi, f_Z)$, $M = BT\sqrt{\rho d_T} / 2$, and \cref{eq:regret_under_q,eq:kl_bound} gives
    \begin{align}
        \mathbb{E}_P[R_T(\pi, f_Z)]
        & \geq \frac{TB\sqrt{\rho d_T}}{8} - \frac{TB\sqrt{\rho d_T}}{2}\cdot \frac{1}{8} \\
        & = \frac{\sigma}{16} \sqrt{T d_T},
    \end{align}
    which implies \Cref{lem:regret_high_ds}
    since $\sup_{f\in\mathcal{F}_k(B)}\mathbb{E}[R_T(\pi, f)]$ is at least $\mathbb{E}_{Z\sim \nu_T}\mathbb{E}[R_T(\pi, f_Z)]$ and $\pi$ is arbitrary.
\end{proof}

\section{Extension to General Compact Domains}
\label{sec:compact_X}

We extend \Cref{thm:lower_bound_finite_X_d_T} to general (possibly continuous) compact domains through a finite approximation of $\mathcal{X}$.  All proofs are given in the supplementary material.

\subsection{Setup}
\label{sec:setup_compact}

We assume $\mathcal{X} \subseteq \mathbb{R}^d$ is nonempty and compact, and let $k: \mathcal{X}\times\mathcal{X} \to \mathbb{R}$ be a non-constant, continuous, and positive semidefinite kernel satisfying $\sup_{x\in\mathcal{X}} k(x,x) \leq 1 $.
The rest of the setup is similar to \Cref{sec:setup_finite}, with some minor differences described as follows.
We reuse the definitions $\mathcal{H}_k, \phi(x), B, \mathcal{F}_k(B)$ and again assume that $f \in \mathcal{F}_k(B)$. We also reuse the definition 
$\Delta_k := \max_{x,x'\in \mathcal{X}} \norm{\phi(x) - \phi(x')}_{k} $, which is well-defined for compact $\mathcal{X}$ (see \Cref{lem:Delta_k_make_sense}). Since $k$ is non-constant, it follows that $\Delta_k > 0$.
The adaptive learning framework $T,\pi,x_t,y_t,\xi_t,\sigma$ and regret definitions $R_T(\pi,f), R^*_T$ remain the same; the max operation in $R_T(\pi, f) := \sum_{t=1}^T (\max_{x\in\mathcal{X}} f(x) - f(x_t) ) $ is justified because $\mathcal{X}$ is compact and $f$ is continuous (see \Cref{lem:f_continuous}). 
We reuse $\lambda := \sigma^2/B^2$,
and extend the definition of the maximum information gain in \cref{eq:gamma_T}: For $\mathcal{A} \subseteq \mathcal{X}$ that is nonempty and compact, we define the maximum information gain on $\mathcal{A}$ by
\begin{align}
    \gamma_T(\mathcal{A})
    & :=  \frac{1}{2} \max_{x_{1:T} \in \mathcal{A}^T} \log\det (I_T + \lambda^{-1} K_T(x_{1:T}) ), \label{eq:gamma_T_extended}
\end{align}
where $K_T(x_{1:T}) = [k(x_s,x_{s'})]_{s,s'\in[T]} $.  
The max operation in \cref{eq:gamma_T_extended} is justified because the mapping $x_{1:T} \mapsto \frac{1}{2} \log\det (I_T + \lambda^{-1} K_T ) $ is continuous and $\mathcal{A}^{T} $ is compact. It follows that $\gamma_T(\mathcal{X}) = \gamma_T$.

\paragraph{Finite Approximation}
Compared to \Cref{sec:setup_finite}, the main ingredient for handling compact $\mathcal{X}$ is a finite approximation of $\mathcal{X}$ shown in the following result.
\begin{proposition}
    \label{prop:finite_approximation}
    Consider the setup in \Cref{sec:setup_compact} where $\mathcal{X}$ is compact.
    Then, there exists a nonempty finite $\tilde{\mathcal{X}}\subseteq \mathcal{X}$ that satisfies
    \begin{align}
        \min_{x'\in\tilde{\mathcal{X}}} \norm{\phi(x) - \phi(x')}_k \leq \frac{\sigma}{4B\sqrt{T}}
        \qquad
        \forall x \in \mathcal{X}, \label{eq:finite_approximation}
    \end{align}
    and
        $\gamma_T(\tilde{\mathcal{X}})
        = \gamma_T(\mathcal{X})$.
\end{proposition}

The $\tilde{\mathcal{X}}$ satisfying \Cref{prop:finite_approximation} is not unique, 
but for our purposes it suffices to fix any such $\tilde{\mathcal{X}}$.
With $\tilde{\mathcal{X}}$ defined, we reuse the notations $E$, $n$, $(e_i)_{i=1}^n$, $I_E$, $\det_E$, $\tr_E$, $\Delta(\mathcal{X})$, $\nu$, $M_\nu$, $\otimes$, $\rho$, $\nu_T$, $M_T$, $d_T$, $(\lambda_i)_{i=1}^n$ for the effective dimension from \Cref{sec:setup_finite}, with the understanding that every occurrence of $\mathcal{X}$ therein is replaced by $\tilde{\mathcal{X}}$. We sometimes write $d_T(\tilde{\mathcal{X}})$ to make the distinction from $d_T(\mathcal{X})$ explicit.

\subsection{Main Results}
\begin{theorem}
    \label{thm:lower_bound_compact_X_d_T}
    Under the setup of \Cref{sec:setup_compact} where $\mathcal{X}$ is compact,
    for every finite $\tilde{\mathcal{X}}$ satisfying \Cref{prop:finite_approximation},
    \begin{align*}
        R^*_T
        \geq \min \{1, 2\Delta_k \} \cdot \frac{\sigma}{32} \sqrt{Td_T (\tilde{\mathcal{X}})},
    \end{align*}
    where $\Delta_k := \max_{x,x'\in \mathcal{X}} \norm{\phi(x) - \phi(x')}_{k} $.
\end{theorem}

Although \Cref{thm:lower_bound_compact_X_d_T} depends on $\tilde{\mathcal{X}}$, 
we can remove such dependence by converting $d_T(\tilde{\mathcal{X}})$ to $\gamma_T(\mathcal{X})$, and obtain the same scaling as \Cref{cor:lower_bound_finite_X_gamma_T}.
\begin{corollary}
    \label{cor:lower_bound_copmact_X_gamma_T}
    Under the setup of \Cref{sec:setup_compact} where $\mathcal{X}$ is compact, we have
    \begin{align*}
        R_T^{*} 
        \geq \min \{1, 2\Delta_k \} \cdot \frac{\sqrt{2}\sigma}{32} \sqrt{\frac{T\gamma_T(\mathcal{X})}{1+\log(1+B^2T/\sigma^2)}},
    \end{align*}
    where $\Delta_k := \max_{x,x'\in \mathcal{X}} \norm{\phi(x) - \phi(x')}_{k} $.
\end{corollary}

The proofs are given in Appendix \ref{app:proof_compact}, and are similar to those for finite $\mathcal{X}$, except for \Cref{lem:off_net_bound}, which extends \cref{eq:fact_1} in \Cref{lem:facts} to every $x \in \mathcal{X}$ with a slightly weaker bound.
We also present a high-probability bound (\Cref{cor:high_prob_lower_bound}) in Appendix \ref{app:high_prob_lower_bound}, which states that for every policy $\pi$, there exists $f \in \mathcal{F}_k(B)$ such that $R_T(\pi, f)$ satisfies the same lower bound as
\Cref{cor:lower_bound_copmact_X_gamma_T} with probability at least $1/4$.

\section{Refinements and Comparisons}
\label{sec:discussion}

To apply our lower bounds and compare them with existing results, we consider the setup of \Cref{sec:setup_compact} with $\mathcal{X} = [0,1]^d $.
We will mainly discuss the implications of the lower bound stated with the maximum information gain $\gamma_T$ (\Cref{cor:lower_bound_copmact_X_gamma_T}) instead of the effective dimension $d_T$ (\Cref{thm:lower_bound_compact_X_d_T}).
Following \citep{Lee25}, we consider six types of kernels, including the squared exponential kernel $k_{\mathrm{SE}}$, rational-quadratic kernel $k_{\mathrm{RQ},a}$ where $a > 0$, Mat\'ern kernel $k_{\nu}$ where $\nu > 0$, $\gamma$-exponential kernel $k_{\gamma-\mathrm{Exp}}$ where $\gamma \in (0,2)$, piecewise-polynomial kernel $k_{\mathrm{PP},q}$ where $q \in \mathbb{Z}_{\geq 0}$ and $q \geq 1$ for $d = 1,2$, and the one-dimensional Dirichlet kernel $k_{\mathrm{Dir}, n}$ where $n \in \mathbb{Z}_{+}$. The details of these kernels are given in \Cref{app:kernel}.

The main findings of this section are summarized in \Cref{tab:gamma_T_rates}. The second column presents sharp rates (in $\Theta(\cdot)$) for $\gamma_T$, which follow by combining results from \citep{Lee25,Janz26}, as discussed in Appendix \ref{app:derive_gamma_T_rates}. 
The third column presents upper bounds of $R^*_T$, which are obtained from the RIPS algorithm (\cite{Camilleri21}, Theorem 3) and the MOSS algorithm (\cite{Audibert09,Lattimore20}, Theorem 9.1), and discussed in \Cref{sec:compare_with_upper_bound}.
The fourth column presents lower bounds of $R^*_T$ based on \Cref{cor:lower_bound_copmact_X_gamma_T} (for which we are not aware of rigorously established stronger lower bounds), and is discussed in \Cref{sec:remove_log}.
\Cref{tab:gamma_T_rates} shows that, for all kernels, the gap between bounds is at most $\sqrt{\log T}$, and that the bounds are matching for the Mat\'ern kernel with $\nu \in (0,2)$, $\gamma$-exponential kernel with $\gamma \in (0,2)$, and piecewise-polynomial kernel with $q = 0, d\geq 3$ or $q = 1$.
As such, our \Cref{tab:gamma_T_rates} sharpens \citep[Table 1]{Lee25}, 
where the regret bounds were expressed in the $\widetilde{O}(\cdot), \widetilde{\Omega}(\cdot)$ sense, and the lower bounds for the rational quadratic kernel and Dirichlet kernel were missing.

\begin{remark}
    Another consequence of \Cref{cor:lower_bound_copmact_X_gamma_T} is that any regret upper bound implies an information gain upper bound.  Specifically, if $R^*_T \leq U_T$, then $\gamma_T = O(U_T^2 \log (T) / T)$. In Appendix \ref{app:upper_bound_on_gamma_T}, we also present a refined upper bound of this kind that can remove the log factor.  These results may be of independent interest for kernels where $\gamma_T$ is not yet established but $R^*_T$ can be bounded by other means.
\end{remark}

\begin{table*}[!ht]
    \caption{
    Rates of the maximum information gain $\gamma_T$ (in $\Theta(\cdot)$) and the 
    bounds on minimax regret $R^*_T$ for six kernels on $\mathcal{X} = [0,1]^d$. For all six kernels, the bounds either match or have a $\sqrt{\log T}$ gap. Lower bounds that are new or improved are highlighted in blue. The kernels are defined in \Cref{app:kernel}.
    }
    \label{tab:gamma_T_rates}
    \begin{center}
    \begin{tabular}{llll}
        \toprule
        Kernel & Rate of $\gamma_T$ & Upper Bound of $R^*_T$ & Lower Bound of $R^*_T$ \\
        \midrule
        Mat\'ern                & $T^{d/(2\nu+d)}$
                                & $\begin{cases}
                                    O(\sqrt{T\gamma_T}), & 0<\nu<2,\\
                                    O(\sqrt{T\gamma_T\log T}), & \nu\ge 2
                                  \end{cases}$
                                & $\Omega(\sqrt{T\gamma_T})$ \\
        Squared exponential     & $\frac{(\log T)^{d+1}}{(\log\log T)^d} $ & $O(\sqrt{T\gamma_T})$  &  \textcolor{blue}{$\Omega(\sqrt{T\gamma_T/\log T})$}  \\
        Rational-quadratic      & $(\log T)^{d+1} $  & $O(\sqrt{T\gamma_T})$ &  \textcolor{blue}{$\Omega(\sqrt{T\gamma_T/\log T})$} \\
        $\gamma$-exponential    & $T^{d / (\gamma + d)} $  & $O(\sqrt{T\gamma_T})$ & $\Omega(\sqrt{T\gamma_T})$ \\
        Piecewise-polynomial    & $T^{d/(2q+1+d)}$
                                & $\begin{cases}
                                    O(\sqrt{T\gamma_T}), & q=0, d\geq 3\ \mathrm{or}\ q = 1,\\
                                    O(\sqrt{T\gamma_T\log T}), & q\ge 2
                                  \end{cases}$
                                & $\Omega(\sqrt{T\gamma_T})$ \\
        Dirichlet ($d =1$)      & $\log T$   & $O(\sqrt{T\gamma_T})$ & \textcolor{blue}{$\Omega(\sqrt{T\gamma_T/\log T})$} \\
        \bottomrule
    \end{tabular}
    \end{center}
\end{table*}

\subsection{When can the Log Factor be Removed?}
\label{sec:remove_log}

When $\mathcal{X}$ is fixed and compact, \Cref{cor:lower_bound_copmact_X_gamma_T} gives $R^*_T = \Omega\bigl( \sqrt{T\gamma_T/\log T} \bigr) $.
We first show that the $\log T$ term therein cannot always be removed,
then discuss a condition on $\gamma_T$ that removes this term, and finally apply this condition to the kernels in \Cref{tab:gamma_T_rates}.

\paragraph{Lack of General $\Omega(\sqrt{T\gamma_T})$ Lower Bound.}

We consider the setting in \Cref{sec:setup_finite} with $\mathcal{X}$ being fixed and finite, and define $K:= \abs{\mathcal{X}} $.
Since $k$ is non-constant, $K \geq 2$.
This can be viewed as a multi-armed bandit problem with Gaussian noise and a constraint that $\norm{f}_k \leq B$. For every $x \in \mathcal{X}$, the reproducing property and Cauchy-Schwarz show that $\abs{f(x)} \leq \abs{\langle f, \phi(x) \rangle_k} \leq \norm{f}_k \norm{\phi(x)}_k \leq B $.
Then, applying the MOSS algorithm (\cite{Lattimore20}, Theorem 9.1) gives
\begin{align}
    R^*_T 
    \leq 39 \sigma \sqrt{KT} + 2BK
    = O(\sqrt{T}). \label{eq:moss}
\end{align}
Fix any sequence $x_1,\ldots, x_T \in \mathcal{X}$.
Since $\abs{\mathcal{X}} = K$, the Gram matrix $K_T(x_{1:T})$ satisfies $\mathrm{rank}(K_T(x_{1:T})) \leq \dim \mathrm{span} \{ \phi(x): x \in \mathcal{X} \} \leq K$, and we list its nonzero eigenvalues and append zeros to obtain the $K$ values $\lambda_1,\ldots, \lambda_K$. Notice that $\sum_{i=1}^K \lambda_i = \tr(K_T(x_{1:T})) = \sum_{t=1}^T k(x_t,x_t) \leq T $. Then,
\begin{multline}
    \log \det (I_T + \lambda^{-1}K_T(x_{1:T})) 
    = \sum_{j=1}^K \log (1 +  \lambda_j/\lambda )
    \leq  K \log \biggl(1 + \sum_{j=1}^K \lambda_j/(K\lambda) \biggr) \leq  K \log \bigl(1 + T/(K\lambda) \bigr),
\end{multline}
where the first inequality applies Jensen's inequality. 
Multiplying by $1/2$ and taking the maximum over $x_{1:T} \in \mathcal{X}^T$ gives $\gamma_T(\mathcal{X}) = O(\log T)$.
On the other hand, \Cref{lem:gamma_T_log_T_lower_bound} in \Cref{app:technical_lemmas} gives $\gamma_T = \Omega(\log T)$, and hence $\gamma_T(\mathcal{X}) = \Theta(\log T)$. \Cref{lem:regret_2} shows that $R^*_T = \Omega(\sqrt{T})$, which combined with \cref{eq:moss} and $\gamma_T = \Theta(\log T)$ gives
\begin{align}
    R^*_T 
    = \Theta(\sqrt{T})
    = \Theta(\sqrt{T\gamma_T/\log T}).
\end{align}
Consequently, an $\Omega(\sqrt{T\gamma_T})$ lower bound is impossible in general, and the log factor is unavoidable.

\paragraph{A Sufficient Condition for $\Omega(\sqrt{T\gamma_T})$ Regret.}

We now regard $\gamma_T$ as the function $\gamma_{(\cdot)}$ evaluated at $T$, accordingly using $\gamma_{m}$ below. 
We start by presenting a lower bound stated in terms of the difference of the maximum information gain.  Here and subsequently, all omitted proofs are given in \Cref{app:proofs_discussion}.

\begin{lemma}
    \label{lem:lower_bound_diff_gamma}
    Consider the setup of \Cref{sec:setup_compact} where $\mathcal{X}$ is compact.
    For every $T \geq 2$ and every integer $1 \leq m < T$, we have
    \begin{align}
        R^*_T \geq  c_1 \cdot \sqrt{ \frac{T(\gamma_T - \gamma_{m}) }{(1+\lambda^{-1}) \sum_{t = m+1}^T t^{-1} }},
    \end{align}
    where $c_1 = \min \bigl\{ \frac{\sigma}{32}, \frac{\sigma \Delta_k}{16}  \bigr\}$.    
\end{lemma}

\Cref{lem:lower_bound_diff_gamma} implies the following sufficient condition for $\Omega(\sqrt{T\gamma_T})$ regret.

\begin{proposition}
    \label{prop:sufficient_condition_to_remove_log}
    Consider the setup of \Cref{sec:setup_compact} with $\mathcal{X}$ being fixed and compact.
    Suppose that there exists a fixed integer $L \geq 2$ and a constant $c \in (0,1)$ such that
    \begin{align}
        \gamma_{\lfloor T/L \rfloor} \leq (1-c) \gamma_T \label{eq:growth_condition}
    \end{align}
    for all sufficiently large $T$. Then $R^*_T = \Omega(\sqrt{T\gamma_T})$.
\end{proposition}

For $\mathcal{X} = [0,1]^d$, \Cref{tab:gamma_T_rates} shows that for the Mat\'ern, $\gamma$-exponential and piecewise-polynomial kernels, $\gamma_T = \Theta(T^{\alpha})$ for some $\alpha > 0$. It is straightforward to verify that these kernels satisfy \Cref{prop:sufficient_condition_to_remove_log}:
since $\gamma_T = \Theta(T^{\alpha}) $, there exist constants $a, b > 0$ such that $aT^{\alpha} \leq \gamma_T \leq bT^{\alpha} $ for all sufficiently large $T$. Then we choose a fixed integer $L \geq 2$ satisfying $\frac{b}{a} L^{-\alpha} \leq \frac{1}{2}$.
Therefore, for sufficiently large $T$,
\begin{align}
    \frac{\gamma_{\lfloor T/L \rfloor}}{\gamma_T}
    \leq \frac{b}{a} \biggl( \frac{\lfloor T/L \rfloor}{T} \biggr)^{\alpha}
    \leq \frac{b}{a} L^{-\alpha}
    \leq \frac{1}{2}.
\end{align}
Then \cref{eq:growth_condition} holds with $c = 1/2$, and \Cref{prop:sufficient_condition_to_remove_log} gives $R^*_T = \Omega(\sqrt{T\gamma_T}) = \Omega(T^{(\alpha + 1)/2} ) $.

On the other hand, the squared exponential, rational-quadratic and Dirichlet kernels fail \cref{eq:growth_condition}.  
To see this, suppose by contradiction that \cref{eq:growth_condition} holds with $L \geq 2$ and $c \in (0,1)$. Then we choose $N$ sufficiently large that $\gamma_N > 0$ and \cref{eq:growth_condition} holds for every $T \geq N$. 
We now fix $j\in \mathbb{Z}_{+}$. Iterating \cref{eq:growth_condition} gives
\begin{align}
    \gamma_{L^jN}
    \geq (1-c)^{-1} \gamma_{L^{j-1}N}
    \geq \cdots
    \geq (1-c)^{-j} \gamma_{N}.
\end{align}
Define $\eta := \log((1-c)^{-1})/\log L > 0 $. It follows that $\gamma_{L^jN} \geq \gamma_N N^{-\eta} (L^jN)^{\eta} $, which contradicts the poly-logarithmic rates of $\gamma_T$ for these kernels in \Cref{tab:gamma_T_rates}.
Failing \cref{eq:growth_condition} does not eliminate the possibility of $\Omega(\sqrt{T\gamma_T})$ regret for these kernels, but proving (or disproving) it would require a different argument.

\subsection{Comparison with Upper Bounds}
\label{sec:compare_with_upper_bound}

In this section, we first derive a general $O(\sqrt{T\gamma_T \log T})$ upper bound from \citep[Theorem 3]{Camilleri21}, then apply two approaches to remove the $\log T$ term.

Since upper bounds of $R^*_T$ are often stated on finite domains \citep{Camilleri21, Li22}, 
we extend such bounds to $\mathcal{X} = [0,1]^d$ as follows.
Let $\tilde{\mathcal{X}} \subset \mathcal{X}$ be finite. 
Define
\begin{align}
    \epsilon(\tilde{\mathcal{X}}) 
    := \sup_{f\in\mathcal{F}_k(B)} \Big(\max_{x\in\mathcal{X}} f(x) - \max_{a\in\tilde{\mathcal{X}} } f(a) \Big). \label{eq:epsilon_x}
\end{align}
Then, running a policy on $\tilde{\mathcal{X}}$ introduces at most an extra $T\epsilon(\tilde{\mathcal{X}})$ term to the minimax regret on $\mathcal{X}$.

\paragraph{A General $O(\sqrt{T\gamma_T \log T})$ Upper Bound.}

For finite $\mathcal{X}$, \citep[Theorem 3]{Camilleri21} presents a high-probability upper bound on $R_T(\pi_{\mathrm{RIPS}}, f) $, where $\pi_{\mathrm{RIPS}}$ corresponds to its RIPS algorithm. This bound leads to the following minimax regret upper bound.

\begin{lemma}[\cite{Camilleri21}, Theorem 3]
    \label{lem:cam_thm_3}
    Consider the setup of \Cref{sec:setup_finite} where $\mathcal{X}$ is finite. Then
    \begin{align*}
        R^*_T 
        = O \biggl( \log(\abs{\mathcal{X}}T) + \sqrt{T\max_{\mathcal{V} \subseteq \mathcal{X}} h(\mathcal{V}, \rho)}  \sqrt{\log(\abs{\mathcal{X}} T)} \biggr),
    \end{align*}
    where, for each nonempty $\mathcal{V}\subseteq \mathcal{X}$, $h(\mathcal{V}, \rho) = \inf_{\nu \in \Delta(\mathcal{V})} \max_{x\in\mathcal{V}} \langle \phi(x), (M_{\nu} + \rho I_E)^{-1} \phi(x) \rangle_k $.
\end{lemma}

The following lemma connects $h(\mathcal{V}, \rho)$ to $\gamma_T$.

\begin{lemma}
    \label{lem:connect_fV_to_gamma_T}
    Consider the setup of \Cref{sec:setup_finite} where $\mathcal{X}$ is finite. 
    For nonempty $\mathcal{V} \subseteq \mathcal{X}$,
    define $h(\mathcal{V}, \rho)$ as in \Cref{lem:cam_thm_3}. Then
    \begin{align}
        \max_{\mathcal{V} \subseteq \mathcal{X}} h(\mathcal{V}, \rho) \leq 2(1 + \lambda^{-1}) \gamma_T.
    \end{align}
\end{lemma}

Combining \Cref{lem:cam_thm_3,lem:connect_fV_to_gamma_T} gives $R^*_T = O \bigl( \log(\abs{\mathcal{X}}T) + \sqrt{T\gamma_T\log(\abs{\mathcal{X}} T)} \bigr)$ when $\mathcal{X}$ is finite; this is within a $\log T$ factor of our general lower bound under the mild condition $|\mathcal{X}| \le {\rm poly}(T)$ (e.g., $|\mathcal{X}| = T^d$).    
To extend to $\mathcal{X} = [0,1]^d $, 
we show that a H\"{o}lder condition allows us to choose a finite $\tilde{\mathcal{X}}\subset \mathcal{X}$ such that $T\epsilon(\tilde{\mathcal{X}})$ is lower-order, which leads to $O(\sqrt{T\gamma_T\log T})$ regret.

\begin{theorem}
    \label{thm:regret_upper_bound_on_cube}
    Consider the setup in \Cref{sec:setup_compact} with $\mathcal{X} = [0,1]^d $. 
    Suppose that there exist fixed $C, \alpha > 0$, such that
    \begin{align}
        \norm{\phi(x) - \phi(x')}_k
        \leq C \norm{x-x'}_2^{\alpha}
        \qquad
        \forall x,x' \in \mathcal{X}. \label{eq:holder_condition}
    \end{align}
    Then, $R^*_T = O \bigl( \sqrt{T\gamma_T \log T} \bigr)$.
\end{theorem}

Importantly, all six kernels in \Cref{tab:gamma_T_rates} satisfy \cref{eq:holder_condition}, so \Cref{thm:regret_upper_bound_on_cube} applies to each of them. The details are given in Appendix \ref{app:kernel_satisfy_holder_condition}.

\paragraph{Removing the Log Factor for Specific Kernels.}

The generic $O(\sqrt{T\gamma_T\log T})$ bound can be sharpened. 
Specifically, we now show that for the squared exponential (SE), rational quadratic, and one-dimensional Dirichlet kernel, the $\log T$ factor can be removed completely by a sharper analysis of $h(\mathcal{V}, \rho)$.

\begin{lemma}
    \label{lem:f_V_sharper}
    Consider the setup of \Cref{sec:setup_finite} with finite $\mathcal{X} \subset [0,1]^d$. 
    For every nonempty $\mathcal{V}\subseteq \mathcal{X}$, 
    define $h(\mathcal{V}, \rho)$ as in \Cref{lem:cam_thm_3}. 
    Then the following upper bounds hold for $\max_{\mathcal{V} \subseteq \mathcal{X}}h(\mathcal{V}, \rho)$:
    \begin{enumerate}
        \item SE kernel $k_{\mathrm{SE}}$: $ O( (\log T/\log\log T)^d )$.
        \item Rational quadratic kernel $k_{\mathrm{RQ}, a}$: $ O( (\log T)^d )$.
        \item One-dimensional Dirichlet kernel $k_{\mathrm{Dir},n}$: $ 2n+1$.
    \end{enumerate}
    The implicit constants are independent of $\mathcal{X}$.
\end{lemma}

For the squared exponential kernel, combining \Cref{lem:cam_thm_3,lem:f_V_sharper} gives
\begin{align*}
    R^*_T 
    = O \biggl( \log(\abs{\mathcal{X}}T) + \sqrt{T \biggl(\frac{\log T}{\log\log T}\biggr)^d }  \sqrt{\log(\abs{\mathcal{X}} T)} \biggr).
\end{align*}
Following the same discretization in \Cref{thm:regret_upper_bound_on_cube} gives
\begin{align}
    R^*_T 
    = O \biggl(\sqrt{T (\log T)^{d+1} /(\log\log T)^d }  \biggr),
\end{align}
which combined with the rate of $\gamma_T$ from \Cref{tab:gamma_T_rates} gives $R^*_T = O(\sqrt{T\gamma_T})$. 
We can similarly verify the $O(\sqrt{T\gamma_T})$ regret bound for the other two kernels.

For the Mat\'ern, $\gamma$-exponential, and piecewise-polynomial kernels, the same approach applies, but fails to remove the $\log T$ factor. Alternatively, we apply MOSS \citep[Theorem 9.1]{Lattimore20} with a careful discretization of $\mathcal{X}$ 
to remove the $\log T$ factor in certain parameter regimes.

\begin{lemma}
    \label{lem:remove_log_MOSS}
    Consider the setup of \Cref{sec:setup_compact} where $\mathcal{X} = [0,1]^d$. For the Mat\'ern kernel with $\nu \in (0,2)$, $\gamma$-exponential kernel with $\gamma \in (0,2)$, and piecewise-polynomial kernel with $q = 0, d\geq 3$ or $q =1$, we have
    \begin{align}
        R^*_T = O(\sqrt{T\gamma_T}).
    \end{align}
\end{lemma}

For the kernel settings in \Cref{lem:remove_log_MOSS}, the upper bounds match the corresponding lower bounds, establishing $R^*_T = \Theta(\sqrt{T\gamma_T})$. For the remaining kernel settings in \Cref{app:kernel}, a $\sqrt{\log T}$ gap remains.

\section{Conclusion}
In this paper, we presented an $\Omega(\sqrt{T\gamma_T/\log T})$ minimax regret lower bound for non-constant continuous positive semidefinite kernels on compact domains of $\mathbb{R}^d$, thus establishing that the optimal minimax regret scaling is $\sqrt{T \gamma_T}$ up to log factors in very high generality.  We further showed that the logarithmic term cannot be removed in general, but can be removed on $[0,1]^d$ if $\gamma_T = \Theta(T^{\alpha})$ for some $\alpha > 0$.  By comparing to suitable upper bounds, we deduced that the minimax regret scaling is precisely $\Theta(\sqrt{T\gamma_T})$ for several kernels of interest, while small gaps (typically $\sqrt{\log T}$ or $\log T$) remain between the bounds for other kernels.

\subsection*{AI Use Disclosure}

{\bf Overview.} We used AI to develop theoretical results, including proposing claims, providing proof outline and details, and proof-reading (but not direct editing). AI also helped with literature search, including summarizing existing results and identifying open problems.
All AI-assisted work has been closely reviewed. In particular, all AI-assisted proofs were checked line by line by the authors. The authors take responsibility for all results and their correctness.

{\bf Specific details.} The main contribution of this paper, i.e., the $\Omega(\sqrt{T\gamma_T/\log T})$ lower bound, emerged from interactions between the authors and the AI agent, where the authors led the interactions. Subsequent interactions between the authors and the AI agent led to a significantly simplified argument that now essentially comprises \Cref{sec:finite,sec:compact_X}.
For \Cref{sec:discussion}, the authors planned the high-level structure, and discussed with AI for details.
For example, the authors asked the AI agent to help derive the rate of $\gamma_T$ for the rational quadratic kernel, because this process is technically involved.
The authors derived $O(\sqrt{T\gamma_T\log T})$ regret from \citep[Theorem 3]{Camilleri21}, while asking the AI agent to check whether an existing stronger upper bound exists. The sharpened $O(\sqrt{T\gamma_T})$ upper bounds also followed from discussion with the AI agent.

\subsection*{Acknowledgement}

This work is supported by the National Research Foundation (NRF) Singapore and the Ministry of Digital Development and Information (MDDI) under the AI Visiting Professorship (Award AIVP-2024-003).

\bibliographystyle{apalike}
\bibliography{reference}

@inproceedings{Alaoui15,
 author = {Alaoui, Ahmed and Mahoney, Michael},
 booktitle = {Advances in Neural Information Processing Systems},
 pages = {},
 title = {Fast Randomized Kernel Ridge Regression with Statistical Guarantees},
 volume = {28},
 year = {2015}
}

@article{Allen21,
  title={Near-optimal discrete optimization for experimental design: A regret minimization approach},
  author={Allen-Zhu, Zeyuan and Li, Yuanzhi and Singh, Aarti and Wang, Yining},
  journal={Mathematical Programming},
  volume={186},
  number={1},
  pages={439--478},
  year={2021},
}

@inproceedings{Audibert09,
  title={Minimax policies for adversarial and stochastic bandits},
  author={Audibert, Jean-Yves and Bubeck, S{\'e}bastien},
  booktitle={Conference on Learning Theory},
  pages={217--226},
  year={2009}
}

@book{Axler24,
    title     = {Linear Algebra Done Right},
    author    = {Axler, Sheldon},
    year      = {2024},
    publisher = {Springer Cham},
    edition   = {4},
}

@article{Auer02a,
author = {Auer, Peter and Cesa-Bianchi, Nicol{\`o} and Freund, Yoav and Schapire, Robert E.},
title = {The Nonstochastic Multiarmed Bandit Problem},
journal = {SIAM Journal on Computing},
volume = {32},
number = {1},
pages = {48-77},
year = {2002},
}

@inproceedings{Cai21,
  title={On lower bounds for standard and robust {G}aussian process bandit optimization},
  author={Cai, Xu and Scarlett, Jonathan},
  booktitle={International Conference on Machine Learning},
  pages={1216--1226},
  year={2021},
}

@inproceedings{Camilleri21,
  title={High-dimensional experimental design and kernel bandits},
  author={Camilleri, Romain and Jamieson, Kevin and Katz-Samuels, Julian},
  booktitle={International Conference on Machine Learning},
  pages={1227--1237},
  year={2021},
}

@inproceedings{Chowdhury17,
  title={On kernelized multi-armed bandits},
  author={Chowdhury, Sayak Ray and Gopalan, Aditya},
  booktitle={International Conference on Machine Learning},
  pages={844--853},
  year={2017},
}

@InProceedings{Derezinski20,
  title = 	 {Bayesian experimental design using regularized determinantal point processes},
  author =       {Derezinski, Michal and Liang, Feynman and Mahoney, Michael},
  booktitle = 	 {International Conference on Artificial Intelligence and Statistics},
  pages = 	 {3197--3207},
  year = 	 {2020},
  volume = 	 {108},
}

@misc{DLMF,
         author = {{DLMF}},
         year = {2026},
       title = "{\it NIST Digital Library of Mathematical Functions}",
howpublished = "\url{https://dlmf.nist.gov/}, Release 1.2.8 of 2026-09-15",
         url = "https://dlmf.nist.gov/",
        note = "F.~W.~J. Olver, A.~B. {Olde Daalhuis}, D.~W. Lozier, B.~I. Schneider,
                R.~F. Boisvert, C.~W. Clark, B.~R. Miller, B.~V. Saunders,
                H.~S. Cohl, and M.~A. McClain, eds."}

@inproceedings{Iwazaki26,
title={Tighter Regret Lower Bound for {G}aussian Process Bandits with Squared Exponential Kernel in Hypersphere},
author={Shogo Iwazaki},
booktitle={International Conference on Machine Learning},
year={2026},
url={https://openreview.net/forum?id=CPB16KCuNk}
}

@article{Janz26,
  title={Maximum effective dimension and information gain},
  author={Janz, David and Akhavan, Arya and Tsybakov, Alexandre B},
  journal={arXiv preprint arXiv:2608.24450},
  year={2026}
}

@book{Lattimore20, 
    place={Cambridge}, 
    title={Bandit algorithms}, 
    publisher={Cambridge University Press}, 
    author={Lattimore, Tor and Szepesvári, Csaba}, 
    year={2020}
}

@article{Lee25,
  title={Consequences of kernel regularity for bandit optimization},
  author={Lee, Madison and Javidi, Tara},
  journal={arXiv preprint arXiv:2512.05957},
  year={2025}
}

@inproceedings{Li22,
  title={Gaussian process bandit optimization with few batches},
  author={Li, Zihan and Scarlett, Jonathan},
  booktitle={International Conference on Artificial Intelligence and Statistics},
  pages={92--107},
  year={2022},
}

@InProceedings{Liu21,
  title = 	 { Smooth Bandit Optimization: Generalization to {H}older Space },
  author =       {Liu, Yusha and Wang, Yining and Singh, Aarti},
  booktitle = 	 {International Conference on Artificial Intelligence and Statistics},
  pages = 	 {2206--2214},
  year = 	 {2021},
  volume = 	 {130},
}

@article{Ma26,
  title={Batched Kernelized Bandits: Refinements and Extensions},
  author={Ma, Chenkai and Chen, Keqin and Scarlett, Jonathan},
  journal={arXiv preprint arXiv:2603.12627},
  year={2026}
}

@book{Rasmussen05,
    author = {Rasmussen, Carl Edward and Williams, Christopher K. I.},
    title = {Gaussian Processes for Machine Learning},
    publisher = {The MIT Press},
    year = {2005},
}

@article{Salgia21,
  title={A domain-shrinking based {B}ayesian optimization algorithm with order-optimal regret performance},
  author={Salgia, Sudeep and Vakili, Sattar and Zhao, Qing},
  journal={Conference on Neural Information Processing Systems},
  volume={34},
  pages={28836--28847},
  year={2021}
}

@inproceedings{Scarlett17,
  title={Lower bounds on regret for noisy {G}aussian process bandit optimization},
  author={Scarlett, Jonathan and Bogunovic, Ilija and Cevher, Volkan},
  booktitle={Conference on Learning Theory},
  pages={1723--1742},
  year={2017},
}

@InProceedings{Singh21,
  title = 	 { Continuum-Armed Bandits: A Function Space Perspective },
  author =       {Singh, Shashank},
  booktitle = 	 {International Conference on Artificial Intelligence and Statistics},
  pages = 	 {2620--2628},
  year = 	 {2021},
}

@inproceedings{Srinivas10,
author = {Srinivas, Niranjan and Krause, Andreas and Kakade, Sham and Seeger, Matthias},
title={Gaussian process optimization in the bandit setting: No regret and experimental design},
year = {2010},
booktitle = {International Conference on Machine Learning},
pages = {1015–1022},
numpages = {8},
}

@inproceedings{Vakili21b,
  title={On information gain and regret bounds in {G}aussian process bandits},
  author={Vakili, Sattar and Khezeli, Kia and Picheny, Victor},
  booktitle={International Conference on Artificial Intelligence and Statistics},
  pages={82--90},
  year={2021},
}

@inproceedings{Valko13,
author = {Valko, Michal and Korda, Nathan and Munos, R\'emi and Flaounas, Ilias and Cristianini, Nello},
title={Finite-time analysis of kernelised contextual bandits},
year = {2013},
booktitle = {Conference on Uncertainty in Artificial Intelligence},
pages = {654–663},
numpages = {10},
}

@book{Wendland04, 
place={Cambridge}, 
series={Cambridge Monographs on Applied and Computational Mathematics}, 
title={Scattered Data Approximation}, 
publisher={Cambridge University Press}, 
author={Wendland, Holger}, 
year={2004}, 
collection={Cambridge Monographs on Applied and Computational Mathematics}
}

\clearpage
\appendix

\section{Definitions of Kernels}
\label{app:kernel}

Following \citep{Lee25}, we consider six stationary kernels. For $x,x' \in \mathcal{X}$, we define $r := \norm{x-x'}_2$, and denote $k(r):=k(x,x')$.
For each kernel, we present the expression and specify the range of parameters, where $\ell > 0$ (when it appears) is a lengthscale parameter.

\begin{itemize}
    \item Squared exponential (SE) kernel:
    \begin{align}
        k_{\mathrm{SE}}(r) = \exp\bigl( - r^2/(2\ell^2) \bigr). \label{eq:k_SE}
    \end{align}
    \item Rational-quadratic kernel:
    \begin{align}
        k_{\mathrm{RQ},a}(r) = \biggl( 1 + \frac{r^2}{2a\ell^2} \biggr)^{-a}, \label{eq:k_RQ}
    \end{align}
    where $a > 0$.
    \item Mat\'ern kernel:
    \begin{align}
        k_{\nu}(r) = \frac{2^{1-\nu}}{\Gamma(\nu)} \biggl( \frac{\sqrt{2\nu}r}{\ell} \biggr)^{\nu} K_\nu \biggl( \frac{\sqrt{2\nu} r}{\ell} \biggr), \label{eq:k_matern}
    \end{align}
    where $\nu > 0$ is a smoothness parameter, $\Gamma$ is the gamma function, and $K_{\nu}$ is the modified Bessel function of the second kind.
    \item $\gamma$-exponential kernel:
    \begin{align}
        k_{\gamma-\mathrm{Exp}}(r) = \exp \bigl( - (r/\ell)^{\gamma} \bigr), \label{eq:k_gamma_exp}
    \end{align}
    where $\gamma \in (0, 2)$. The special case $\gamma = 2$ is excluded because that recovers the SE kernel.
    \item Piecewise-polynomial kernel:
    \begin{align}
        k_{\mathrm{PP},q}(r) = 
        \begin{cases}
            c_{0,q} + \sum_{j=1}^{\lfloor d/2 \rfloor + 3q + 1} c_{j,q} r^j, & 0 \leq r \leq 1, \\
            0, & r > 1.
        \end{cases}
        \label{eq:k_pp}
    \end{align}
    Here $d$ is the dimension of $\mathcal{X}$, $q \in \mathbb{Z}_{\geq 0} $ and $q \geq 1$ if $d = 1,2$. \citep{Lee25} states that the coefficients $c_{j,q}$ can be computed recursively (\cite{Wendland04}, Theorem 9.13).
    \item The one-dimensional Dirichlet kernel:
    \begin{align}
        k_{\mathrm{Dir},n}(r)
        = \frac{1}{2n+1} \sum_{k=-n}^n \exp(-ikr), \label{eq:k_Dirichlet}
    \end{align}
    where $n \in \mathbb{Z}_{+}$. We exclude $n = 0$ because that leads to $k\equiv 1$, while \Cref{sec:setup_compact} requires the kernel to be non-constant, and we restrict to $d = 1$ because the Dirichlet kernel may fail to be positive semidefinite in higher dimensions.
\end{itemize}

\section{Technical Lemmas}
\label{app:technical_lemmas}

\begin{lemma}
    \label{lem:trace_to_inner_product}
    Consider the setup of \Cref{sec:setup_finite} where $\mathcal{X}$ is finite. Let $A$ be a self-adjoint operator on $E$ and $u \in E$.
    Then
    \begin{align}
        \tr_E ((u\otimes u) A) = \langle u, Au \rangle_k.
    \end{align}
\end{lemma}
\begin{proof}
    Recall that $e_1,\ldots e_n$ is an orthonormal basis of $E$. Then
    \begin{align}
        \tr_E ((u\otimes u) A) 
        = \sum_{i=1}^n \langle (u\otimes u) A e_i, e_i \rangle_k
        = \sum_{i=1}^n \langle u, A e_i \rangle_k \langle u, e_i \rangle_k
        = \sum_{i=1}^n \langle u, e_i \rangle_k \langle Au, e_i \rangle_k
        = \langle u, Au \rangle_k,
    \end{align}
    where the last equality follows from \citep[Eq.~6.30(c)]{Axler24}.
\end{proof}

\begin{lemma}
    \label{lem:operator_with_otimes}
    Consider the setup of \Cref{sec:setup_finite} where $\mathcal{X}$ is finite. Let $A$ be a self-adjoint operator on $E$ and $u \in E$.
    Then
    \begin{align}
        (Au) \otimes (Au) = A(u \otimes u) A.
    \end{align}
\end{lemma}
\begin{proof}
    Choose any $v \in E$. Then $\bigl((Au) \otimes (Au)\bigr) v = (Au) \langle Au, v \rangle_k = (Au) \langle u, Av \rangle_k = \bigl(A (u\otimes u) A\bigr)v $.
\end{proof}

\begin{lemma}
    \label{lem:positive_definite}
    Consider the setup of \Cref{sec:setup_finite} where $\mathcal{X}$ is finite. For every $\nu \in \Delta(\mathcal{X})$,
    \begin{align}
        M_{\nu} + \rho I_E 
        \succ M_{\nu}
        \succeq 0. \label{eq:psd_claims}
    \end{align}
    Consequently, $(M_{\nu} + \rho I_E)^{-1} \succ 0$, $(M_{\nu} + \rho I_E)^{-1/2}$ is self-adjoint, and $(M_{\nu} + \rho I_E)^{-2} \preceq \rho^{-1} (M_{\nu} + \rho I_E)^{-1} $.
\end{lemma}
\begin{proof}
    To show $M_{\nu} + \rho I_E \succ M_{\nu}$, it suffices to show that $\rho I_E \succ 0$ and $M_{\nu} \succeq 0$.  The condition $\rho I_E \succ 0$ holds since $\langle \rho I_E v, v \rangle_k = \rho \norm{v}_k^2 > 0 $ for every $v \in E$ and $v \neq 0$.
    To show $M_{\nu} \succeq 0$, we recall that $M_{\nu} = \sum_{x\in\mathcal{X}} \nu(x) \phi(x)\otimes \phi(x)$, such that $\langle M_{\nu} v, v \rangle_k = \sum_{x\in\mathcal{X}}\nu(x) \abs{\langle \phi(x), v \rangle_k}^2 \geq 0 $ for every $v \in E$.  This establishes \cref{eq:psd_claims}.

    Since $M_{\nu} + \rho I_E$ is self-adjoint, it follows that $(M_{\nu} + \rho I_E)^{-1}$ is self-adjoint, which combined with the fact that all eigenvalues of $(M_{\nu} + \rho I_E)^{-1}$ are positive proves $(M_{\nu} + \rho I_E)^{-1} \succ 0$. This also proves $(M_{\nu} + \rho I_E)^{-1/2}$ is self-adjoint.
    To show $(M_{\nu} + \rho I_E)^{-2} \preceq \rho^{-1} (M_{\nu} + \rho I_E)^{-1} $, observe that $\rho^{-1} (M_{\nu} + \rho I_E)^{-1} -(M_{\nu} + \rho I_E)^{-2} = \rho^{-1} (M_{\nu} + \rho I_E)^{-1}M_{\nu} (M_{\nu} + \rho I_E)^{-1} $.
    For every $u \in E$, observe that $\langle (M_{\nu} + \rho I_E)^{-1}M_\nu (M_{\nu} + \rho I_E)^{-1} u, u \rangle_k = \langle M_\nu (M_{\nu} + \rho I_E)^{-1} u, (M_{\nu} + \rho I_E)^{-1}u \rangle_k \geq 0 $, which proves $(M_{\nu} + \rho I_E)^{-1}M_{\nu} (M_{\nu} + \rho I_E)^{-1} \succeq 0$, and thus $(M_{\nu} + \rho I_E)^{-2} \preceq \rho^{-1} (M_{\nu} + \rho I_E)^{-1}$.
\end{proof}

\begin{lemma}[\cite{Camilleri21}, Lemma 15]
    \label{lem:camilleri_lem_15}
    Consider the setup of \Cref{sec:setup_finite} where $\mathcal{X}$ is finite.
    Then
    \begin{align}
        \max_{x\in \mathcal{X}} \bigl\langle \phi(x), (M_T + \rho I_E)^{-1}\phi(x) \bigr\rangle_k
        = \sum_{x\in \mathcal{X}} \nu_T(x) \bigl\langle \phi(x), (M_T + \rho I_E)^{-1}\phi(x) \bigr\rangle_k
        = d_T.
    \end{align}
\end{lemma}

\begin{lemma}
    \label{lem:bound_by_two_d_T}
    Consider the setup of \Cref{sec:setup_finite} where $\mathcal{X}$ is finite. For every $\nu \in \Delta(\mathcal{X})$,
    \begin{align}
        \tr_E(M_{\nu} (M_{\nu} + \rho I_{E})^{-1} )
        \leq 2d_T.
    \end{align}
\end{lemma}
\begin{proof}
    Fix $\nu \in \Delta(\mathcal{X})$. 
    Since $M_T := M_{\nu_T} $ where $\nu_T \in \Delta(\mathcal{X})$,
    \Cref{lem:positive_definite} shows that $M_{\nu}, M_T \succeq 0 $, such that
    \begin{align}
        \tr_E (M_{\nu}(M_{\nu}+ \rho I_E)^{-1} )
        & \leq \tr_E ((M_{\nu} + M_T) (M_{\nu} + M_T+ \rho I_E)^{-1} ) \\
        & \leq \tr_E ((M_{\nu} + M_T) (M_T+ \rho I_E)^{-1} ) \\
        & = \tr_E (M_{\nu} (M_T+ \rho I_E)^{-1} ) + d_T, \label{eq:one_d_T}
    \end{align}
    where the equality follows from the linearity of the trace and \cref{eq:effective_dimension}. For the first term in \cref{eq:one_d_T}, observe that
    \begin{align}
        \tr_E (M_{\nu} (M_T+ \rho I_E)^{-1} )
        & = \sum_{x \in \mathcal{X}} \nu(x) \tr_E ((\phi(x)\otimes \phi(x)) (M_T+ \rho I_E)^{-1} ) \\
        & = \sum_{x \in \mathcal{X}} \nu(x) \langle \phi(x), (M_T+ \rho I_E)^{-1}\phi(x) \rangle_k \label{eq:use_trace_ideneity} \\
        & \leq \max_{x\in\mathcal{X}} \langle \phi(x), (M_T+ \rho I_E)^{-1}\phi(x) \rangle_k \label{eq:mean_by_max} \\
        & = d_T, \label{eq:use_lem_15}
    \end{align}
    where \cref{eq:use_trace_ideneity} applies \Cref{lem:trace_to_inner_product}, \cref{eq:mean_by_max} follows since the average is upper bounded by the maximum, and \cref{eq:use_lem_15} applies \Cref{lem:camilleri_lem_15}.
    Combining \cref{eq:one_d_T,eq:use_lem_15} proves the result.
\end{proof}

\begin{lemma}
    \label{lem:det_ratio}
    Consider the setup of \Cref{sec:setup_finite} where $\mathcal{X}$ is finite. Fix $u \in E$, and let $A$ be an operator on $E$ such that $A - u\otimes u \succ 0 $. Then
    \begin{align}
        \frac{\det_E (A - u\otimes u) }{\det_E A}
        = 1 - \langle u, A^{-1}u \rangle_k.
    \end{align}
\end{lemma}
\begin{proof}
    It is straightforward to verify that $u\otimes u \succeq 0$, such that $A \succ 0$. 
    Define $w := A^{-1/2} u$.
    Then $1 - \langle u, A^{-1}u \rangle_k = 1 - \norm{w}_k^2 $.
    Factoring out $A^{1/2}$ gives
    \begin{align}
        \frac{\det_E (A - u\otimes u) }{\det_E A}
        = \det_E \bigl(I_E -  A^{-1/2} (u\otimes u) A^{-1/2}  \bigr)
        = \det_E (I_E -  w\otimes w  ), \label{eq:remove_one_w}
    \end{align}
    where the second equality follows from $A^{-1/2}$ being self-adjoint and \Cref{lem:operator_with_otimes}.
    Now we consider two cases of $w$ in \cref{eq:remove_one_w}:
    If $w \neq 0$, then the operator $I_E -  w\otimes w$ has eigenvalue $1 - \norm{w}^2_k $ on $\mathrm{span}\{w\}$, and eigenvalue $1$ on $\mathrm{span}\{w\}^{\perp} $, such that $\det_E (I_E -  w\otimes w  ) = 1 - \norm{w}_k^2 $.
    If $w = 0$, then $\det_E (I_E -  w\otimes w  ) = \det_E I_E = 1 = 1 - \norm{w}_k^2$.
    Then $1 - \langle u, A^{-1}u \rangle_k = \det_E (I_E -  w\otimes w  )$, which combined with \cref{eq:remove_one_w} proves the result.
\end{proof}

\begin{lemma}
    \label{lem:det_ratio_2}
    Consider the setup of \Cref{sec:setup_finite} where $\mathcal{X}$ is finite. Fix $u \in E$, and let $A$ be an operator on $E$ such that $A \succ 0 $. Then
    \begin{align}
        \frac{\det_E (A + u\otimes u) }{\det_E A}
        = 1 + \langle u, A^{-1}u \rangle_k.
    \end{align}
\end{lemma}
\begin{proof}
    The proof is entirely analogous to that of \Cref{lem:det_ratio}, so is omitted. 
\end{proof}

\begin{lemma}
    \label{lem:same_eigenvalues}
    Consider the setup of \Cref{sec:setup_finite} where $\mathcal{X}$ is finite. Fix $x_1,\ldots, x_T \in \mathcal{X}$, and define $\hat{\nu} := \frac{1}{T}\sum_{s=1}^T \delta_{x_s} \in \Delta(\mathcal{X}) $. Then $K_T(x_{1:T})$ and $TM_{\hat{\nu}} $ have the same positive eigenvalues (including multiplicities).
\end{lemma}
\begin{proof}
    Define the linear map $\Phi: \mathbb{R}^T \to E $ by $\Phi b_i = \phi(x_i) $, where $b_i$ is the $i$-th element of the standard basis of $\mathbb{R}^T$.
    By \cref{eq:moment_operator} and the definition of $K_T(x_{1:T})$,
    we have $K_T = \Phi^*\Phi$ and $TM_{\hat{\nu}} = \Phi \Phi^*$, where $\Phi^*$ is the adjoint of $\Phi$. Since $K_T(x_{1:T}), M_{\hat{\nu}} \succeq 0$, all their non-zero eigenvalues are positive, and the result follows.
\end{proof}

\begin{lemma}
    \label{lem:gamma_T_log_T_lower_bound}
    Consider the setup in \Cref{sec:setup_compact} where $\mathcal{X}$ is compact. Then,
    \begin{align}
        \gamma_T(\mathcal{X}) 
        \geq \frac{1}{2}\log (1+Tk(x_0,x_0)\lambda^{-1}),
    \end{align}
    where $x_0 \in \argmax_{x\in\mathcal{X}} k(x,x)$.
\end{lemma}
\begin{proof}
    For every $x \in \mathcal{X}$, we have $k(x,x) = \langle \phi(x), \phi(x) \rangle_k \geq 0 $.
    Since $k$ is non-constant, we have $k(x_0,x_0) > 0$.
    By the definition of $\gamma_T(\mathcal{X})$ in \cref{eq:gamma_T_extended} and considering the choice $x_1 = \dotsc = x_T = x_0$, we have
    \begin{align}
        \gamma_T(\mathcal{X})
        \geq \frac{1}{2}\log \det (I_T + \lambda^{-1} k(x_0,x_0) 1_T)
        = \frac{1}{2}\log (1+Tk(x_0,x_0)\lambda^{-1}),
    \end{align}
    where $1_T$ is the $T\times T$ all-ones matrix. 
\end{proof}

\begin{lemma}
    \label{lem:d_T_is_positive}
    Under the setup in \Cref{sec:setup_finite} where $\mathcal{X}$ is finite, we have $d_T > 0$.
\end{lemma}
\begin{proof}
    By \citep[Lemma 15]{Camilleri21}, we have
    \begin{align}
        d_T
        = \max_{x\in \mathcal{X}} \bigl\langle \phi(x), (M_T + \rho I_E)^{-1}\phi(x) \bigr\rangle_k.
    \end{align}
    Since $M_T + \rho I_E$ is positive definite by \Cref{lem:positive_definite}, it follows that $(M_T + \rho I_E)^{-1}$ is also positive definite.
    Observe that there exists $x_0 \in \mathcal{X}$ such that $\phi(x_0) \neq 0$, because otherwise $k(x,x') = \langle \phi(x), \phi(x') \rangle_k = 0$ for every $x,x' \in \mathcal{X}$, contradicting that $k$ is non-constant. 
    Then $d_T \geq \bigl\langle \phi(x_0), (M_T + \rho I_E)^{-1}\phi(x_0) \bigr\rangle_k > 0$.
\end{proof}

\begin{lemma}
    \label{lem:unique_operator}
    Consider the setup of \Cref{sec:setup_finite} where $\mathcal{X}$ is finite.
    Suppose $\nu_1,\nu_2\in\Delta(\mathcal{X})$, and that both $\nu_1,\nu_2$ are regularized $D$-optimal designs satisfying \cref{eq:d_optimal_design} for some regularization level $\rho > 0$. Then $M_{\nu_1} = M_{\nu_2} $.
\end{lemma}
\begin{proof}
    If $\nu_1 = \nu_2$, then the result trivially follows, so we assume that $\nu_1 \neq \nu_2$.
    Fix $\nu \in \Delta(\mathcal{X})$. By \Cref{lem:positive_definite}, $M_\nu + \rho I_E$ is positive definite.
    We now verify that $\log\det_E $ is strictly concave in the space of positive definite operators. Let $A,B$ be distinct positive definite operators on $E$, and let $t \in (0,1)$. Define $C := A^{-1/2}BA^{-1/2} $. Then $C$ is positive definite. Let $\lambda_1,\ldots, \lambda_n > 0$ be its eigenvalues. Since $(1-t) A + tB = A^{1/2}\bigl( (1-t)I_E + tC \bigr)A^{1/2} $, it follows that
    \begin{align}
        \log\det_E\bigl( (1-t)A + tB \bigr)
        = \log\det_E A + \sum_{j=1}^n\log \bigl( (1-t) + t\lambda_j \bigr).
    \end{align}
    Also notice that $\log\det_E(B) = \log\det_E(A) + \sum_{i=1}^n \log\lambda_i $. Since the logarithm is strictly concave on $(0, \infty)$, we have $\log\bigl( (1-t) + t\lambda_i \bigr) \geq t\log \lambda_i $, with equality when $\lambda_i = 1$. Since $A \neq B$, it follows that $C \neq I_E$, such that at least one of $\lambda_1,\ldots, \lambda_n$ does not equal $1$. Combining, we have
    \begin{align}
        \log\det_E\bigl( (1-t)A + tB \bigr)
        > (1-t) \log\det_E (A) + t \log \det_E(B). \label{eq:strong_concave}
    \end{align}
    Now suppose by contradiction that $M_{\nu_1} \neq M_{\nu_2}$. Consider the mixed design $\nu_t := (1-t)\nu_1 + t\nu_2 \in \Delta(\mathcal{X}) $. Since the mapping $\nu \mapsto M_\nu$ is linear, $M_{\nu_t} = (1-t) M_{\nu_1} + t M_{\nu_2} $.
    Then, applying \cref{eq:strong_concave} with $A = M_{\nu_1} + \rho I_E$ and $B = M_{\nu_2} + \rho I_E$ shows that
    \begin{align}
        \log\det_E\bigl( M_{\nu_t} + \rho I_E \bigr)
        > (1-t) \log\det_E (M_{\nu_1} + \rho I_E) + t \log \det_E(M_{\nu_2} + \rho I_E),
    \end{align}
    which implies the design value at $\nu_t$ (notice $\nu_t$ is different from $\nu_1$ and $\nu_2$ since $\nu_1 \neq \nu_2$) is larger than the $D$-optimal design value, establishing a contradiction.
    Hence, $M_{\nu_1} = M_{\nu_2}$.
\end{proof}

\begin{lemma}
    \label{lem:trace_of_M_T}
    Consider the setup of \Cref{sec:setup_finite} where $\mathcal{X}$ is finite. Then
    \begin{align}
        \mathrm{tr}_E(M_{T}) \leq 1.
    \end{align}
\end{lemma}
\begin{proof}
    Recall that $e_1,\ldots, e_n$ is an orthonormal basis of $E$.
    Then
    \begin{align}
        \mathrm{tr}_E(M_{T}) 
        & = \sum_{x\in\mathcal{X}}\nu_T(x) \tr_E\bigl(\phi(x)\otimes \phi(x)\bigr) \\
        & = \sum_{x\in\mathcal{X}}\nu_T(x) \sum_{i=1}^n \langle e_i, (\phi(x)\otimes \phi(x)) e_i \rangle_k \\
        & = \sum_{x\in\mathcal{X}}\nu_T(x) \sum_{i=1}^n \abs{\langle \phi(x), e_i \rangle_k}^2 \label{eq:use_otimes_def} \\
        & = \sum_{x\in\mathcal{X}}\nu_T(x) \norm{\phi(x)}_k^2 \label{eq:parseval} \\
        & = \sum_{x\in\mathcal{X}}\nu_T(x) k(x,x) \label{eq:use_reproduce_property} \\
        & \leq 1, \label{eq:tr_M_T}
    \end{align}
    where \cref{eq:use_otimes_def} follows from the definition of $\otimes$,
    \cref{eq:parseval} uses Parseval's identity \citep[Eq.~6.30(b)]{Axler24}, \cref{eq:use_reproduce_property} follows from the reproducing property, and \cref{eq:tr_M_T} follows since $k(x,x) \leq 1$ for every $x \in \mathcal{X}$.
\end{proof}

\begin{lemma}
    \label{lem:phi_continuous}
    Under the setup of \Cref{sec:setup_compact} where $\mathcal{X}$ is compact,
    the mapping $x \mapsto \phi(x)$ from $\mathcal{X}$ to $\mathcal{H}_k$ is continuous.
\end{lemma}
\begin{proof}
    Observe that $\norm{\phi(x) - \phi(x')}_{k} = \sqrt{\langle \phi(x) - \phi(x'), \phi(x) - \phi(x') \rangle_k} = \sqrt{k(x,x) + k(x',x') - 2k(x,x')} $. Since $k$ is continuous, $\norm{\phi(x) - \phi(x')}_{k} \to 0$ as $x \to x'$, which implies that the mapping $x \to \phi(x)$ is continuous.
\end{proof}

\begin{lemma}
    \label{lem:f_continuous}
    Consider the setup of \Cref{sec:setup_compact} where $\mathcal{X}$ is compact.
    For any $f \in \mathcal{F}_k(B)$,
    the function
    $f: \mathcal{X} \to \mathbb{R}$ is continuous.
\end{lemma}
\begin{proof}
    By the reproducing property, $f(x) = \langle f, \phi(x) \rangle_k$. It follows that $\abs{f(x) - f(x')} = \abs{\langle f, \phi(x) - \phi(x') \rangle_k} \leq \norm{f}_k \norm{\phi(x) - \phi(x')}_k \leq B\norm{\phi(x) - \phi(x')}_k$, where the first inequality uses Cauchy-Schwarz, and the second follows from $\norm{f}_k \leq B$. Combined with \Cref{lem:phi_continuous}, this shows that $f(x) \to f(x') $ as $x \to x'$, which implies that $f$ is continuous on $\mathcal{X}$.
\end{proof}

\begin{lemma}
    \label{lem:Delta_k_make_sense}
    Under the setup of \Cref{sec:setup_compact}, $\Delta_k := \max_{x,x'\in \mathcal{X}} \norm{\phi(x) - \phi(x')}_{k} $ is well-defined.
\end{lemma}
\begin{proof}
    Fix $x,x'\in \mathcal{X}$.
    Since $k$ is continuous, it follows that $\norm{\phi(x) - \phi(x')}_k = \sqrt{k(x,x)+k(x',x') - 2k(x,x')} $ is continuous on $\mathcal{X}\times \mathcal{X}$. Also observe that $\mathcal{X}^2$ is compact. Then the max operation in $\Delta_k$ is well-defined.
\end{proof}

\begin{lemma}[\cite{Auer02a}, Proof of Lemma A.1]
    \label{lem:change_of_measure}
    Let $P, Q$ be two probability measures on the same measurable space, and let $H$ be a random variable satisfying $0 \leq H \leq M$ almost surely. Then 
    \begin{align}
        \mathbb{E}_P[H] \geq \mathbb{E}_Q[H] - M \sqrt{\frac{1}{2}D_{\mathrm{KL}}(Q\Vert P) }.
    \end{align}
\end{lemma}

\section{Proofs of Theorem \ref{thm:lower_bound_finite_X_d_T} and Corollary \ref{cor:lower_bound_finite_X_gamma_T} (Finite Domains)}
\label{app:proof_finite}

Recall that $V_T := M_T  + \rho I_E $, and $u_x:= d_T^{-1/2} V_T^{-1/2} \phi(x) \in \mathcal{H}_k$ for every $x \in \mathcal{X}$. $u_x$ can be viewed as a normalized (and transformed) feature vector; its key properties are given as follows.

\begin{lemma}
    \label{lem:facts}
    Let $Z \sim \nu_{T}$ for $\nu_T$ given in \cref{eq:d_optimal_design}. Then
        \begin{align}
            \norm{u_x}_k & \leq 1 \quad \forall x \in \mathcal{X}, \label{eq:fact_1} \\
            \norm{u_Z}_k & = 1 \quad \text{almost surely}, \label{eq:fact_2} \\
            \mathbb{E}[u_Z \otimes u_Z] & \preceq d_T^{-1} I_E, \label{eq:fact_3} \\
            \norm{\mathbb{E}[u_Z]}_k & \leq d_T^{-1/2}. \label{eq:fact_4}
            \end{align}
\end{lemma}
\begin{proof}
    \textbf{Proof of \cref{eq:fact_1}.}
    Since $V_T^{-1/2}$ is self-adjoint (\Cref{lem:positive_definite}), we have
    \begin{align}
        \norm{u_x}_k^2 = d_T^{-1} \bigl\langle \phi(x), V_T^{-1}\phi(x) \bigr\rangle_k, \label{eq:u_x_norm}
    \end{align}
    which combined with \Cref{lem:camilleri_lem_15} proves \cref{eq:fact_1}.

    \textbf{Proof of \cref{eq:fact_2}.}
    Applying \cref{eq:u_x_norm} and then \Cref{lem:camilleri_lem_15} gives
    \begin{align}
        \mathbb{E} [\norm{u_Z}_k^2]
        & = \sum_{x\in\mathcal{X}} \nu_T(x) \norm{u_x}_k^2 \\
        & = \sum_{x\in\mathcal{X}} \nu_T(x) d_T^{-1} \bigl\langle \phi(x), V_T^{-1}\phi(x) \bigr\rangle_k \\
        & = 1. \label{eq:exp_u_Z}
    \end{align}
    By \cref{eq:fact_1}, each $x \in \mathcal{X}$ satisfies $\norm{u_x}_k^2 \leq 1$, which combined with \cref{eq:exp_u_Z} shows that $\norm{u_x}_k^2 = 1$ for every $x \in \mathcal{X}$ such that $\nu_T(x) > 0$, proving \cref{eq:fact_2}.

    \textbf{Proof of \cref{eq:fact_3}.}
    Since $V_T^{-1/2}$ is self-adjoint (\Cref{lem:positive_definite}), \Cref{lem:operator_with_otimes} gives
    \begin{align}
        (u_Z \otimes u_Z)
        & = d_T^{-1} V_T^{-1/2} \bigl(\phi(Z) \otimes \phi(Z)\bigr) V_T^{-1/2}.
    \end{align}
    Averaging over $Z\sim \nu_T$ then gives
    \begin{align}
        \mathbb{E}_{Z}[u_Z \otimes u_Z]
        & = d_T^{-1} V_T^{-1/2} \mathbb{E}_{Z} [\phi(Z)\otimes \phi(Z)] V_T^{-1/2} \\
        & = d_T^{-1} V_T^{-1/2} M_{T} V_T^{-1/2} \label{eq:use_def_MT} \\
        & = d_T^{-1} \bigl(I_E - \rho V_T^{-1} \bigr) \label{eq:MT_identity} \\
        & \preceq d_T^{-1} I_E, \label{eq:preceq_I_E}
    \end{align}
    where \cref{eq:use_def_MT} follows from the definition of $M_T$ (see \cref{eq:moment_operator}), \cref{eq:MT_identity} follows since $M_T = V_T - \rho I_E $, and
    \cref{eq:preceq_I_E} follows because $V_T^{-1}$ is positive definite (\Cref{lem:positive_definite}).
    
    \textbf{Proof of \cref{eq:fact_4}.}
    We fix $v \in E$, and observe that
    \begin{align}
        \abs{\langle \mathbb{E}[u_Z], v \rangle_k}^2 
        & =  \mathbb{E}[\langle u_Z,v \rangle_k ]^2 \label{eq:lin_exp_1} \\
        & \leq \mathbb{E} [\abs{\langle u_Z, v \rangle_k}^2 ] \label{eq:jensen} \\
        & = \mathbb{E} [\langle v, (u_Z \otimes u_Z) v \rangle_k] \label{eq:def_otimes} \\
        & = \langle v, \mathbb{E}[(u_Z \otimes u_Z)] v \rangle_k \label{eq:lin_exp_2} \\
        & \leq \langle v, d_T^{-1} I_E v \rangle_k \label{eq:prec} \\
        &= d_T^{-1} \norm{v}_k^2, \label{eq:temp}
    \end{align}
    where \cref{eq:lin_exp_1,eq:lin_exp_2} follow from linearity of expectation, 
    \cref{eq:def_otimes} follows because $\langle v, (u_Z\otimes u_Z)v \rangle_k = \langle v, u_Z \langle u_Z,v \rangle_k \rangle_k = \abs{\langle u_Z, v \rangle_k}^2 $,
    \cref{eq:jensen} uses Jensen's inequality, and \cref{eq:prec} follows from \cref{eq:fact_3}.
    If $\mathbb{E}[u_Z] \neq 0$, 
    taking $v = \mathbb{E}[u_Z] / \norm{\mathbb{E}[u_Z]}_k $ in \cref{eq:temp} proves \cref{eq:fact_4}.
    If $\mathbb{E}[u_Z] = 0$, then \cref{eq:fact_4} holds trivially.
\end{proof}

The following lemma uses a simpler two-function construction, and will be used to handle the case $d_T < 4$.

\begin{lemma}
    \label{lem:regret_2}
    Under the setup in \Cref{sec:setup_finite},
    \begin{align}
        R_T^{*}
        \geq \frac{\Delta_k}{8} \min\{ 2BT, \sigma\sqrt{T}\},
    \end{align}
    where $\Delta_k := \max_{x,x'\in \mathcal{X}} \norm{\phi(x) - \phi(x')}_{k} $.
\end{lemma}

\begin{proof}
    \textbf{Step 1: Construct hard functions.}
    Choose $x_{+}, x_{-} \in \mathcal{X} $ attaining $\Delta_k = \max_{x,x'\in\mathcal{X}}\norm{\phi(x) - \phi(x')}_k $, and define
    \begin{align}
        g := \frac{\phi(x_{+}) - \phi(x_{-})}{\Delta_k} \in \mathcal{H}_k , \label{eq:function_g}
    \end{align}
    which implies $\norm{g}_{k} = 1$.
    By the reproducing property, $g(x_{+}) - g(x_{-}) = \langle g, \phi(x_{+}) - \phi(x_{-}) \rangle_k = \Delta_k^{-1} \norm{\phi(x_{+}) - \phi(x_{-})}_k^2 = \Delta_k$.
    Conversely, for every $x, x' \in \mathcal{X}$, Cauchy-Schwarz gives $\abs{g(x) - g(x')} = \abs{\langle g, \phi(x) - \phi(x') \rangle_k} \leq \norm{g}_k \norm{\phi(x) - \phi(x')}_k \leq \Delta_k $. It follows that
    \begin{align}
        \max_{x\in\mathcal{X}} g(x) - \min_{x\in\mathcal{X}} g(x) = \Delta_k.
    \end{align}
    Moreover,
    for every $x \in \mathcal{X}$, $\abs{g(x)} = \abs{\langle g, \phi(x) \rangle_k} \leq \norm{g}_{k} \norm{\phi(x)}_{k} \leq 1$ due to $k(x,x) \leq 1$. Let $a := \min\bigl\{B, \frac{\sigma}{2\sqrt{T}}\bigr\} $, and for $v \in \{+1,-1\} $, define
    \begin{align}
        f_{v} := va\cdot g \in \mathcal{H}_k.
    \end{align}
    Since $a \leq B$,
    it follows that $f_v \in \mathcal{F}_k(B)$.
    
    \textbf{Step 2: Change of measure.}
    Recall that $H_{t} := (X_1,Y_1,\ldots, X_{t},Y_{t}) $ is the history until time $t$.
    Fix a policy $\pi$.
    Let $V$ be uniform on $\{+1,-1\}$.
    Let $P$ be the joint law on $(V,H_T)$ obtained by drawing $V$ and running $\pi$ against $f_V$, and $Q$ be the joint law on $(V,H_T)$ obtained by drawing $V$ independently and running $\pi$ against the zero function.
    Applying the chain-rule for KL divergence and following \cref{eq:kl_decompose} to \cref{eq:KL_gaussian} gives
    \begin{align}
        D_{\mathrm{KL}}(Q \Vert P)
        & = \frac{1}{2\sigma^2} \sum_{t=1}^T \mathbb{E}_Q[ f_V(X_t)^2 ] \\
        & = \frac{a^2}{2\sigma^2} \sum_{t=1}^T \mathbb{E}_Q[ g(X_t)^2 ] \\
        & \leq \frac{a^2T}{2\sigma^2} \\
        & \leq \frac{1}{8}, \label{eq:kl_one_eigths}
    \end{align}
    where the first inequality uses $\abs{g(x)} \leq 1$, and the second inequality uses $a \leq \frac{\sigma}{2\sqrt{T}}$.
    Under $Q$, $V$ is independent of $X_t$. Conditioned on $X_t=x_t$, the instantaneous regret averaged over the two signs is
    \begin{align}
        \frac{a}{2} \biggl(\max_{x\in\mathcal{X}}g(x) - g(x_t)\biggr) + \frac{a}{2} \biggl(g(x_t) - \min_{x\in\mathcal{X}}g(x) \biggr) 
        = \frac{a\Delta_k}{2}.
    \end{align}
    Summing over $t \in [T]$ gives $\mathbb{E}_Q[R_T(\pi, f_V)] = \frac{aT\Delta_k}{2}$.
    For either value of $V$, the instantaneous regret lies in $[0, a\Delta_k]$, such that $0 \leq R_T(\pi, f_V) \leq a T\Delta_k$ under both $P$ and $Q$. Applying \Cref{lem:change_of_measure} with $H = R_T(\pi, f_V)$, $M = a\Delta_kT$, \cref{eq:kl_one_eigths} and $\mathbb{E}_Q[R_T(\pi, f_V)] = \frac{aT\Delta_k}{2}$ gives
    \begin{align}
        \mathbb{E}_P[R_T(\pi, f_V)]
        & \geq \frac{aT\Delta_k}{2} - \frac{a\Delta_kT}{4}
        = \frac{a\Delta_kT}{4} \\
        & = \frac{\Delta_k}{8} \min\{ 2BT, \sigma\sqrt{T}\}, \label{eq:RT_second_bound}
    \end{align}
    where the last equality follows from the definition of $a$.
    This implies the desired lower bound,
    since $\sup_{f\in\mathcal{F}_k(B)}\mathbb{E}[R_T(\pi, f)] \geq \mathbb{E}_{V}\mathbb{E}[R_T(\pi, f_V)]$ and $\pi$ is an arbitrary policy.
\end{proof}

The next lemma relates the information gain to the effective dimension.  It is related to \citep[Lemma 16]{Camilleri21}, but that result concerns the opposite direction that the maximum information gain is lower bounded by the effective dimension.

\begin{lemma}
    \label{lem:d_T_bounds}
    Consider the setup in \Cref{sec:setup_finite} where $\mathcal{X}$ is finite, except that $k$ may be constant.
    Then
    \begin{align}
        \frac{2\gamma_T}{1+\log(1+\rho^{-1})}
        \leq
        d_T
        \leq \rho^{-1},
    \end{align}
    where $\rho = \sigma^2/(B^2 T)$.
\end{lemma}
\begin{proof}
    Fix $x_1,\ldots, x_T \in \mathcal{X}$, and define the empirical distribution $\hat{\nu} :=  \frac{1}{T} \sum_{t=1}^T \delta_{x_t} $. 
    \Cref{lem:same_eigenvalues} in \Cref{app:technical_lemmas} shows that $K_T(x_{1:T})$ and $TM_{\hat{\nu}}$ have the same positive eigenvalues (including multiplicity), which implies
    \begin{align}
        \det (I_T + \lambda^{-1} K_T(x_{1:T}))
        = \det_E (I_E + \rho^{-1}  M_{\hat{\nu}}) .
    \end{align}
    Recalling that $n = \dim(E)$, this gives
    \begin{align}
        \log \det (I_T + \lambda^{-1} K_T)
        = \log \det_E (\rho I_E + M_{\hat{\nu}})  - n\log \rho,
    \end{align}
    and hence
    \begin{align}
        \gamma_T
        & = \frac{1}{2} \max_{x_{1:T} \in \mathcal{X}^T} \log\det (I_T + \lambda^{-1} K_T(x_{1:T}) ) \\
        & = \frac{1}{2} \biggl( \max_{\hat{\nu}} \log \det_E (\rho I_E + M_{\hat{\nu}}) - n\log \rho \biggr) \label{eq:max_over_empirical_distribution} \\
        & \leq \frac{1}{2} \biggl( \max_{\nu \in \Delta(\mathcal{X})} \log \det_E (\rho I_E + M_{\nu}) - n\log \rho \biggr) \\
        & = \frac{1}{2} \log \det_E (I_E + \rho^{-1} M_{T}), \label{eq:compare_gamma_T}
    \end{align}
    where the max operation in \cref{eq:max_over_empirical_distribution} is over all empirical distributions on $\mathcal{X}$ from $T$ queries, and the last equality follows from the definition of $M_T$.

    \Cref{lem:trace_of_M_T} in \Cref{app:technical_lemmas} gives $\tr_E(M_T) \leq 1$.    
    Then, the definition of $d_T$ in \cref{eq:effective_dimension} gives $d_T = \sum_{i=1}^n \frac{\lambda_i}{\lambda_i + \rho} \leq \rho^{-1} \sum_{i=1}^n \lambda_i = \rho^{-1} \tr_E(M_T) \leq \rho^{-1}$, proving the desired upper bound.
    For later use, we note that the inequality $\tr_E(M_T) \leq 1$ also implies $\lambda_i \leq 1$ for $i \in [n]$.
    
    It remains to prove the lower bound.
    For $i \in [n]$, define $w_i := \rho^{-1}\lambda_i$, such that $0 \leq w_i \leq \rho^{-1} $. Then, the formula $d_T = \sum_{i=1}^n\frac{\lambda_i}{\lambda_i + \rho}$ equivalently gives $d_T = \sum_{i=1}^n \frac{w_i}{1+w_i} $.
    For $w > 0$, define $q(w):= \frac{1+w}{w}\log(1+w) $, with $q(0) := 1$. Since $q'(w) = \frac{w -\log(1+w)}{w^2} > 0 $, it follows that $q(\cdot)$ is strictly increasing. 
    Then, \cref{eq:compare_gamma_T} implies that
    \begin{align}
        \gamma_T
        & \leq \frac{1}{2} \sum_{i=1}^n \log (1+w_i)
        = \frac{1}{2} \sum_{i=1}^n \frac{w_i}{1+w_i} q(w_i) \\
        & \leq \frac{q(\rho^{-1})}{2} \sum_{i=1}^n \frac{w_i}{1+w_i}
        = \frac{q(\rho^{-1})d_T}{2}. \label{eq:gamma_T_d_T}
    \end{align}
    Since $\log(1+w) \leq w$, it follows that $q(w) = \frac{1}{w}\log(1+w) + \log(1+w) \leq 1 + \log(1+w) $. Then, \cref{eq:gamma_T_d_T} gives
    \begin{align}
        d_T 
        \geq \frac{2\gamma_T}{1+\log(1+\rho^{-1})},
    \end{align} 
    which proves the desired lower bound.
\end{proof}

\paragraph{Completion of the Proof of Theorem \ref{thm:lower_bound_finite_X_d_T} and Corollary \ref{cor:lower_bound_finite_X_gamma_T}}

\begin{proof}[Proof of \Cref{thm:lower_bound_finite_X_d_T}] 
    We split into two cases for $d_T$.

    \textbf{Case 1:} $d_T \geq 4$. In this case, \Cref{lem:regret_high_ds} directly gives $R^*_T \geq \frac{\sigma}{16} \sqrt{Td_T} $.

    \textbf{Case 2:} $d_T < 4$. In this case, $\sigma\sqrt{T} \geq \frac{\sigma}{2} \sqrt{Td_T}$. Moreover, the upper bound in \Cref{lem:d_T_bounds} and $\rho = \sigma^2/(B^2T)$ implies $\sigma\sqrt{T d_T} \leq BT$, which we can crudely weaken to $2BT \geq \frac{\sigma}{2}\sqrt{Td_T}$.
    Applying \Cref{lem:regret_2} gives
    \begin{align}
        R_T^{*}
        \geq \frac{\Delta_k}{8} \min\{ 2BT, \sigma\sqrt{T}\}
        \geq \frac{\Delta_k\sigma}{16} \sqrt{Td_T}.
    \end{align}
\end{proof}

\begin{proof}[Proof of \Cref{cor:lower_bound_finite_X_gamma_T}]
    Combining \Cref{thm:lower_bound_finite_X_d_T} with the lower bound in \Cref{lem:d_T_bounds} proves \Cref{cor:lower_bound_finite_X_gamma_T}.
\end{proof}

\section{Proofs of Theorem \ref{thm:lower_bound_compact_X_d_T} and Corollary \ref{cor:lower_bound_copmact_X_gamma_T}
(General Compact Domains)}
\label{app:proof_compact}

Recall that we fix a set $\tilde{\mathcal{X}}$ satisfying \Cref{prop:finite_approximation}, and that $d_T(\tilde{\mathcal{X}})$ is computed on $\tilde{\mathcal{X}}$. For convenience we write $d_T = d_T(\tilde{\mathcal{X}})$ within the proofs.

We extend $M_T: E \to E$ to $M : \mathcal{H}_k \to \mathcal{H}_k$ such that $M\vert_{E} = M_T $ and $Mv = 0 $ for $v \in E^{\perp} $. We define $I_k$ as the identity map on $\mathcal{H}_k$. 
It follows that $M + \rho I_k$ is positive definite (similar to \Cref{lem:positive_definite}).
Also notice that
\begin{align}
    0 < 2 \gamma_T(\mathcal{X}) = 2\gamma_T(\tilde{\mathcal{X}}) \leq \log \det_E(I_E + \rho^{-1} M_T),
\end{align}
where the first inequality follows because $k$ is non-constant on $\mathcal{X}$, the equality follows from \Cref{prop:finite_approximation}, the second inequality follows from \cref{eq:d_optimal_design} and \Cref{lem:same_eigenvalues}.
Then $M_T$ has a positive eigenvalue, and thus $d_T > 0$.
Then, we generalize the definition of $u_x$ in \cref{eq:ux}, now defined for $x \in \mathcal{X}$ as
\begin{align}
    u_x:= d_T^{-1/2} (M + \rho I_k)^{-1/2} \phi(x) \in \mathcal{H}_k. \label{eq:ux_2}
\end{align}
Applying the decomposition $\mathcal{H}_k = E\oplus E^{\perp} $ to $(M + \rho I_k)^{-1/2}$ gives $(M + \rho I_k)^{-1/2} = (M_T + \rho I_E)^{-1/2} \oplus \rho^{-1/2}I_{E^{\perp}} $. It follows that \cref{eq:ux_2} agrees with \cref{eq:ux} on $\tilde{\mathcal{X}}$. Consequently, we can apply \Cref{lem:facts} with the new definition of $u_x$, where \cref{eq:fact_3} now reads
\begin{align}
    \mathbb{E}[u_Z \otimes u_Z] \preceq d_T^{-1} I_k. \label{eq:fact_3_new}
\end{align}

We start by proving \Cref{prop:finite_approximation}.
\begin{proof}[Proof of \Cref{prop:finite_approximation}]
    We first prove \cref{eq:finite_approximation}.
    Since $\mathcal{X}$ is compact and the mapping $x \mapsto \phi(x)$ is continuous (see \Cref{lem:phi_continuous}), it follows that $\phi(\mathcal{X} ) := \{\phi(x): x \in \mathcal{X}\} \subseteq \mathcal{H}_{k} $ is compact. 
    We fix $\delta := \sigma / (4B\sqrt{T}) > 0$, and consider $\bigcup_{\phi(x) \in \phi(\mathcal{X})} B(\phi(x), \delta)$, 
    where $B(\phi(x), \delta) := \{\phi(x') \in \phi(\mathcal{X}): \norm{\phi(x) - \phi(x')}_k < \delta \} $. It follows that $\bigcup_{\phi(x) \in \phi(\mathcal{X})} B(\phi(x), \delta)$ is an open cover of $\phi(\mathcal{X})$. 
    Compactness implies that there exists a finite subcover, i.e., there exists $\tilde{\mathcal{X}}\subseteq \mathcal{X}$ that is finite and satisfies
    \begin{align}
        \phi(\mathcal{X}) \subseteq \bigcup_{x\in\tilde{\mathcal{X}}} B(\phi(x), \delta).
    \end{align}
    It remains to show that $\tilde{\mathcal{X}}$ can be chosen to satisfy  $\gamma_T(\tilde{\mathcal{X}})
        = \gamma_T(\mathcal{X})$.  We consider $x^*_1,\ldots, x^*_T$, which is a maximizing tuple of $\gamma_T(\mathcal{X})$ in \cref{eq:gamma_T_extended}, and enlarge $\tilde{\mathcal{X}}$ to contain this tuple.
    Then, $\gamma_T(\mathcal{X}) = \frac{1}{2}\log\det(I_T + \lambda^{-1}K_T(x^*_{1:T}) ) \leq \gamma_T(\tilde{\mathcal{X}})$, while $\tilde{\mathcal{X}}\subseteq \mathcal{X}$ gives $\gamma_T(\tilde{\mathcal{X}}) \leq \gamma_T(\mathcal{X})$. Consequently, $\gamma_T(\mathcal{X}) = \gamma_T(\tilde{\mathcal{X}})$.
\end{proof}

The following lemma generalizes \cref{eq:fact_1}.

\begin{lemma}
    \label{lem:off_net_bound}
    Under the setup in \Cref{sec:setup_compact},
    for every $x \in \mathcal{X}$,
    \begin{align}
        \norm{u_x}_k \leq 1 +  \frac{1}{4\sqrt{d_T(\tilde{\mathcal{X}})}}.
    \end{align}
    Consequently, if $d_T(\tilde{\mathcal{X}}) \geq 4$, then $\norm{u_x}_k \leq 9/8$.
\end{lemma}
\begin{proof}
    Fix $x \in \mathcal{X}$. Since $\tilde{\mathcal{X}}$ satisfies \cref{eq:finite_approximation}, there exists $x'\in \tilde{\mathcal{X}}$ such that $\norm{\phi(x) - \phi(x')}_k \leq \frac{\sigma}{4B\sqrt{T}}$.
    For an operator $V$ on $\mathcal{H}_k$, define the operator norm $\norm{V}_{\mathrm{op}} := \sup \{\norm{Vf}_k: f \in \mathcal{H}_k, \norm{f}_k = 1 \} $.
    Observe that $M + \rho I_k \succeq \rho I_k$.
    Then, for every $f \in \mathcal{H}_k$ with $\norm{f}_k = 1$,
    \begin{align}
        \norm{(M + \rho I_k)^{-1/2} f}_k^2
        = \langle f, (M + \rho I_k)^{-1} f \rangle_k
        \leq \rho^{-1} \norm{f}^2_k,
    \end{align}
    which implies $\norm{(M + \rho I_k)^{-1/2}}_{\mathrm{op}} \leq \rho^{-1/2} $.
    Consequently,
    \begin{align}
        \norm{u_x}_k - \norm{u_{x'}}_k
        & = \norm{d_T^{-1/2} (M + \rho I_k)^{-1/2} \phi(x)}_k - \norm{d_T^{-1/2} (M + \rho I_k)^{-1/2} \phi(x')}_k \\
        & \leq d_T^{-1/2} \norm{(M + \rho I_k)^{-1/2} (\phi(x) - \phi(x'))}_k  \\
        & \leq d_T^{-1/2} \norm{(M + \rho I_k)^{-1/2}}_{\mathrm{op}} \cdot \norm{(\phi(x) - \phi(x'))}_k  \\
        & \leq d_T^{-1/2} \rho^{-1/2} \cdot \frac{\sigma}{4B\sqrt{T}}  \\
        & = \frac{1}{4\sqrt{d_T}},
    \end{align}
    Since $x' \in \tilde{\mathcal{X}}$, \cref{eq:fact_1} gives $\norm{u_{x'}}_k \leq 1$.
    Rearranging proves the result.
\end{proof}

Next, we provide the lower bound on regret for $d_T \geq 4$, generalizing \Cref{lem:regret_high_ds}.

\begin{lemma}
    \label{lem:regret_high_dT_compact}
    Under the setup in \Cref{sec:setup_compact},
    if $d_T(\tilde{\mathcal{X}}) \geq 4$,
    then
    \begin{align}
        R_T^{*}
        \geq \frac{\sigma}{32} \sqrt{T d_T(\tilde{\mathcal{X}})}.
    \end{align}
\end{lemma}

\begin{proof}    
    The proof is similar to \Cref{lem:regret_high_ds} in structure, and the main differences are in numerical constants, so we omit details where appropriate.
    
    \textbf{Step 1: Construct hard functions.}
    For every $z \in \mathrm{supp}(\nu_T)$, define
    \begin{align}
        f_z := \frac{B}{6} \sqrt{\frac{\rho}{d_T}} (M_T + \rho I_E)^{-1} \phi(z)
        \in \mathcal{H}_k.
    \end{align}
    It follows that $\norm{f_z}_k^2 \leq \frac{B^2}{36} \leq B^2$, and hence $f_z \in \mathcal{F}_k(B)$.
    In addition, observe that
    \begin{align}
        f_z(x) = \frac{B}{6}\sqrt{\rho d_T} \langle u_z, u_x \rangle_k. \label{eq:fz_at_x}
    \end{align}
    Combining \Cref{lem:off_net_bound} and \cref{eq:fact_1,eq:fact_2} gives
    \begin{align}
        f_z(z) = \frac{B}{6} \sqrt{\rho d_T},
        \qquad
        \abs{f_z(x)} \leq \frac{3B}{16} \sqrt{\rho d_T}
        \qquad
        \forall x \in \mathcal{X}. \label{eq:f_values}
    \end{align}

    \textbf{Step 2: Change of measure.} 
    To use \Cref{lem:change_of_measure}, we first bound the KL divergence.
    Let $H_{t} := (X_1,Y_1,\ldots, X_{t},Y_{t}) $ be the history until time $t$. Fix a policy $\pi$. Let $P$ be the joint law on $(Z,H_T)$ obtained by first drawing $Z \sim \nu_T$, then running $\pi$ against $f_Z$. Let $Q$ be the joint law on $(Z,H_T)$ when the policy is instead run on the zero function, and $Z \sim \nu_T$ is drawn independently from the interaction.
    Then
    \begin{align}
        D_{\mathrm{KL}} (Q \Vert P)
        & = \frac{1}{2\sigma^2} \sum_{t=1}^T \mathbb{E}_Q [f_Z(X_t)^2] \label{eq:kl_chain_rule} \\
        & = \frac{1}{2\sigma^2} \sum_{t=1}^T \mathbb{E}_Q [\mathbb{E}_{Z}[f_Z(X_t)^2]] \label{eq:conditional_independence} \\
        & \leq \frac{T}{2\sigma^2} \cdot \frac{9B^2\rho}{256} \label{eq:bound_eqzfzxt} \\
        & = \frac{9}{512}, \label{eq:kl_bound_2}
    \end{align}
    where \cref{eq:kl_chain_rule} follows from the chain rule for KL divergence,
    \cref{eq:conditional_independence} follows because the random variables $Z$ and $X_t$ are independent under $Q$,
    \cref{eq:bound_eqzfzxt} is explained below, and \cref{eq:kl_bound_2} uses $\rho = \sigma^2/(B^2T)$.  To verify \cref{eq:bound_eqzfzxt}, observe that for every $x \in \mathcal{X}$, $\mathbb{E}_Z[f_Z(x)^2] = \frac{B^2\rho d_T}{36} \langle u_x, \mathbb{E}_Z[u_Z\otimes u_Z]u_x \rangle_k \leq \frac{B^2\rho}{36} \langle u_x, u_x \rangle_k \leq \frac{9B^2\rho}{256}$,
    where the first inequality follows from \cref{eq:fact_3_new}, and the second inequality follows from \Cref{lem:off_net_bound}.
   
    Next, we proceed to bound $\mathbb{E}_Q[R_T(\pi, f_Z)]$.
    For any $x \in \mathcal{X}$, we have
    \begin{align}
        \mathbb{E}_Z[f_Z(x)] =
        \frac{B\sqrt{\rho d_T}}{6} \langle \mathbb{E}_Z[u_Z], u_x \rangle_k
        \leq \frac{B\sqrt{\rho}}{6} \norm{u_x}_k
        \leq \frac{3B\sqrt{\rho}}{16}, \label{eq:EZ_fZ}
    \end{align}
    where the equality follows from \cref{eq:fz_at_x}, the first inequality follows from \cref{eq:fact_4} and Cauchy-Schwarz, and the second inequality follows from \Cref{lem:off_net_bound}.
    The independence of $Z$ and $X_t$ under $Q$ implies $\mathbb{E}_Q [ f_Z(X_t)] = \mathbb{E}_Q [\mathbb{E}_{Z}[f_Z(X_t)]]  $, which combined with \cref{eq:f_values,eq:EZ_fZ} and $d_T \geq 4$ gives
    \begin{align}
        \mathbb{E}_Q[R_T(\pi, f_Z)] 
        = \sum_{t=1}^T \mathbb{E}_Q \bigl[\max_{x}f_Z(x) - f_Z(X_t) \bigr] 
        \geq BT \biggl( \frac{\sqrt{\rho d_T}}{6} - \frac{3\sqrt{\rho}}{16} \biggr)
        \geq \frac{7TB\sqrt{\rho d_T}}{96}. \label{eq:EQ_RT}
    \end{align}
    Finally, \cref{eq:f_values} implies $0 \leq R_T(\pi, f_Z) \leq \frac{3BT\sqrt{\rho d_T}}{8} $.
    Applying \Cref{lem:change_of_measure} with $H = R_T(\pi, f_Z)$, $M = \frac{3BT\sqrt{\rho d_T}}{8}$, and \cref{eq:EQ_RT,eq:kl_bound_2} gives
    \begin{align}
        \mathbb{E}_P[R_T(\pi, f_Z)]
        \geq \frac{7TB\sqrt{\rho d_T}}{96} -  \frac{3BT\sqrt{\rho d_T}}{8} \cdot \frac{3}{32}
        = \frac{29\sigma}{32\cdot24} \sqrt{Td_T}
        \geq \frac{\sigma}{32} \sqrt{T d_T},
    \end{align}
    which proves \Cref{lem:regret_high_dT_compact}
    since $\sup_{f\in\mathcal{F}_k(B)}\mathbb{E}[R_T(\pi, f)] \geq \mathbb{E}_{Z\sim \nu_T}\mathbb{E}[R_T(\pi, f_Z)]$ and the policy $\pi$ is arbitrary.
\end{proof}

\begin{lemma}
    \label{lem:d_T_bound_compact}
    Under the setup of \Cref{sec:setup_compact}, 
    \begin{align}
        \frac{2\gamma_T(\mathcal{X})}{1+\log(1+\rho^{-1})}
        \leq d_T(\tilde{\mathcal{X}})
        \leq \rho^{-1},
    \end{align}
    where $\rho = \sigma^2 / (B^2 T)$.
\end{lemma}
\begin{proof}
    Applying \Cref{lem:d_T_bounds} to $\tilde{\mathcal{X}}$ ($k$ may be constant on $\tilde{\mathcal{X}} \times \tilde{\mathcal{X}}$, but that is allowed in \Cref{lem:d_T_bounds}) proves $\frac{2\gamma_T(\tilde{\mathcal{X}})}{1+\log(1+\rho^{-1})} \leq d_T \leq \rho^{-1}$, which combined with $\gamma_T(\tilde{\mathcal{X}}) = \gamma_T(\mathcal{X})$ from \Cref{prop:finite_approximation} proves the result.
\end{proof}

\paragraph{Completion of the Proof of \Cref{thm:lower_bound_compact_X_d_T} and \Cref{cor:lower_bound_copmact_X_gamma_T}}

\begin{proof}[Proof of \Cref{thm:lower_bound_compact_X_d_T}]
    We again split into two cases for $d_T$.

    \textbf{Case 1:} $d_T \geq 4$.
    In this case, \Cref{lem:regret_high_dT_compact} directly shows that $R^*_T \geq \frac{\sigma}{32} \sqrt{Td_T}$.

    \textbf{Case 2:} $d_T < 4$. Under the setup in \Cref{sec:setup_compact},
    $\Delta_k:= \max_{x,x'\in\mathcal{X}} \norm{\phi(x) - \phi(x')}_k $ and the $x,x'$ attaining $\Delta_k$ are justified (see \Cref{lem:Delta_k_make_sense}), such that \Cref{lem:regret_2} holds for compact $\mathcal{X}$.
    Then $d_T < 4$ implies $\sigma\sqrt{T} \geq \frac{\sigma}{2} \sqrt{Td_T}$, and the upper bound in \Cref{lem:d_T_bound_compact} implies $BT \geq \sigma\sqrt{Td_T}$, such that $2BT \geq \frac{1}{2}\sigma\sqrt{Td_T}$. Then, applying \Cref{lem:regret_2} gives $R^*_T \geq \frac{\Delta_k\sigma}{16}\sqrt{Td_T}$.
\end{proof}

\begin{proof}[Proof of \Cref{cor:lower_bound_copmact_X_gamma_T}]
    Combining \Cref{thm:lower_bound_compact_X_d_T} and \Cref{lem:d_T_bound_compact} proves \Cref{cor:lower_bound_copmact_X_gamma_T}.
\end{proof}

\subsection{High-Probability Lower Bound}
\label{app:high_prob_lower_bound}

While we have focused on the average regret, we can readily adapt our arguments to obtain a high-probability lower bound, which is stated as follows. 

\begin{corollary}
    \label{cor:high_prob_lower_bound}
    Consider the setup of \Cref{sec:setup_compact} where $\mathcal{X}$ is compact. For every randomized policy $\pi$, there exists a fixed $f \in \mathcal{F}_k(B)$, such that with probability at least $1/4$,
    \begin{align}
        R_T(\pi, f) 
        \geq \min \{1, 2\Delta_k \} \cdot \frac{\sqrt{2}\sigma}{32} \sqrt{\frac{T\gamma_T(\mathcal{X})}{1+\log(1+B^2T/\sigma^2)}},
    \end{align}
    where $\Delta_k := \max_{x,x'\in \mathcal{X}} \norm{\phi(x) - \phi(x')}_{k} $.
\end{corollary}

\begin{proof}
    We will first prove a high-probability lower bound expressed using $d_T = d_T(\tilde{\mathcal{X}})$, then transfer to $\gamma_T(\mathcal{X})$ using \Cref{lem:d_T_bound_compact}.
    We consider two cases.

    \textbf{Case 1: $d_T \geq 4$.}
    We follow the setup in \Cref{lem:regret_high_dT_compact}: we consider the hard functions $\{f_z\}$, where $z \in \mathrm{supp}(\nu_T)$. We fix some policy $\pi$, and consider probability measures $P$, $Q$, which are joint laws on $(Z,H_T)$ where $Z \sim \nu_T$, and $H_T$ is the history until time $T$ obtained by running $\pi$ against $f_Z$. 
    Specifically, $P$ is obtained by drawing $Z\sim \nu_T$ and running $\pi$ against $f_Z$, and $Q$ is obtained by drawing $Z$ independently and running $\pi$ against the zero function.
    For convenience we define $S:= \sqrt{\sigma^2 T d_T}$.

    By \cref{eq:f_values}, $f_Z(Z) = \frac{B}{6} \sqrt{\rho d_T}$, which yields
    \begin{align}
        R_T(\pi, f_Z)
        \geq \frac{BT}{6} \sqrt{\rho d_T} - \sum_{t=1}^T f_Z(X_t)
        = \frac{S}{6} - \sum_{t=1}^T f_Z(X_t). \label{eq:R_T_in_S}
    \end{align}
    Moreover, by \cref{eq:fz_at_x}, $f_z(x) = \frac{B}{6}\sqrt{\rho d_T} \langle u_z, u_x \rangle_k$, which gives $\sum_{t=1}^T f_Z(X_t) = \frac{B}{6}\sqrt{\rho d_T} \langle u_Z, \sum_{t=1}^T u_{X_t} \rangle_k $.
    We then observe that
    \begin{align}
        \mathbb{E}_Q \biggl[ \biggl(\sum_{t=1}^T f_Z(X_t)\biggr)^2 \,\bigg\vert\, X_{1:T} \biggr]
        & = \frac{B^2\rho d_T}{36} \mathbb{E}_Q \biggl[  \biggl\langle \sum_{t=1}^T u_{X_t}, u_Z\otimes u_Z \sum_{t=1}^T u_{X_t}\biggr \rangle_k \,\bigg\vert\, X_{1:T} \biggr] \\
        & = \frac{B^2\rho d_T}{36} \biggl\langle \sum_{t=1}^T u_{X_t}, \mathbb{E}[u_Z\otimes u_Z] \sum_{t=1}^T u_{X_t}\biggr \rangle_k, \label{eq:use_conditional_independence} \\
        & \leq \frac{B^2\rho}{36} \biggl\lVert \sum_{t=1}^T u_{X_t} \biggr \rVert_k^2, \label{eq:use_fact_new} \\
        & \leq \frac{B^2\rho}{36}\cdot T \sum_{t=1}^T \norm{u_{X_t}}_k^2 \label{eq:use_cauchy} \\
        & \leq \frac{9\sigma^2 T}{256} \label{eq:use_off_net_bound} \\
        & \leq \frac{9S^2}{1024}, \label{eq:use_def_S}
    \end{align}
    where \cref{eq:use_conditional_independence} follows because $Z$ is independent of the zero-function interaction under $Q$, \cref{eq:use_fact_new} follows from \cref{eq:fact_3_new}, 
    \cref{eq:use_cauchy} applies Cauchy-Schwarz,
    \cref{eq:use_off_net_bound} follows from combining \Cref{lem:off_net_bound} and $\rho = \sigma^2/(B^2T)$, and \cref{eq:use_def_S} applies the definition of $S$ and $d_T \geq 4$.

    We now combine the above findings to lower bound the regret.  
    First observe that if $\sum_{t=1}^T f_Z(X_t) \leq \frac{13S}{96}$, then \cref{eq:R_T_in_S} gives $R_T(\pi, f_Z) \geq \frac{S}{32}$. 
    Moreover, applying Markov's inequality to $\bigl(\sum_{t=1}^T f_Z(X_t)\bigr)^2$ gives
    \begin{align}
        \mathbb{P}_Q(R_T(\pi, f_Z) \geq S / 32 )
        \geq \mathbb{P}_Q\biggl(\sum_{t=1}^T f_Z(X_t) \leq \frac{13S}{96}\biggr)
        \geq 1 - \mathbb{P}_Q\biggl(\biggl(\sum_{t=1}^T f_Z(X_t)\biggr)^2 \geq \biggl(\frac{13S}{96} \biggr)^2 \biggr)
        \geq 1 - \frac{81}{169}
        = \frac{88}{169}.
    \end{align}
    Recall that \cref{eq:kl_bound_2} gives $D_{\mathrm{KL}}(Q \Vert P) \leq \frac{9}{512}$.
    Then, applying \Cref{lem:change_of_measure} with $H = \mathbb{I}_{\{R_T(\pi, f_Z) \geq S/32\}}$ and $M = 1$ gives
    \begin{align}
        \mathbb{P}_P(R_T(\pi, f_Z) \geq S/32)
        \geq \mathbb{P}_Q(R_T(\pi, f_Z) \geq S/32) - \sqrt{\frac{1}{2}D_{\mathrm{KL}}(Q \Vert P)}
        \geq \frac{88}{169} - \frac{3}{32}
        = \frac{2309}{5408}
        \geq \frac{1}{4}.
    \end{align}
    Then, since $S/32 = \sqrt{\sigma^2 T d_T} / 32 \geq \min \{1, 2\Delta_k \} \cdot \frac{\sigma}{32} \sqrt{Td_T}$, we have
    \begin{align}
        P\biggl(R_T(\pi, f_Z) \geq \min \{1, 2\Delta_k \} \cdot \frac{\sigma}{32} \sqrt{Td_T}\biggr)
        \geq \frac{1}{4}.
    \end{align}
    This holds for $Z\sim \nu_T$, and it follows that there exists some fixed $z \in \mathrm{supp}(\nu_T)$ such that $R_T(\pi, f_z) \geq \min \{1, 2\Delta_k \} \cdot \frac{\sigma}{32} \sqrt{Td_T}$ with probability at least $1/4$.

    \textbf{Case 2: $0 < d_T < 4$.}
    Here we follow the setup in \Cref{lem:regret_2}. We define $a := \min\{B, \sigma/(2\sqrt{T}) \} $, and consider the hard functions $f_v := va \cdot g$, where $v \in \{-1,+1\}$ and $g$ is defined in \cref{eq:function_g}. We again fix some policy $\pi$, and consider probability measures $P$, $Q$, which are joint laws on $(V,H_T)$ where $V$ is uniform on $\{-1,+1\}$, and $H_T$ is the history until time $T$ obtained by running $\pi$ against $f_V$. Specifically, $P$ is obtained by drawing $V$ and running $\pi$ against $f_V$, and $Q$ is obtained by drawing $V$ independently and running $\pi$ against the zero function.
    We define $R_{+} := \sum_{t=1}^T (\max_{x\in\mathcal{X}} f_{+1}(x) - f_{+1}(x_t) ) $ and $R_{-} := \sum_{t=1}^T (\max_{x\in\mathcal{X}} f_{-1}(x) - f_{-1}(x_t) ) $, and consider any fixed realization $x_1,\ldots, x_T \in \mathcal{X}$. 
    
    Recall that $\Delta_k := \max_{x,x' \in \mathcal{X}} \norm{\phi(x) - \phi(x')}_k $ and the proof of \Cref{lem:regret_2} gives $\max_{x\in\mathcal{X}} g(x) - \min_{x\in\mathcal{X}} g(x) = \Delta_k $. We observe that
    \begin{align}
        R_{+} + R_{-}
        = a \sum_{t=1}^T \biggl( \max_{x\in\mathcal{X}}g(x) - g(x_t) + g(x_t) - \min_{x\in\mathcal{X}}g(x) \biggr)
        = a T \Delta_k,
    \end{align}
    which implies that at least one of $R_{+}$, $R_{-}$ is at least $aT\Delta_k/2$.  Hence,
    \begin{align}
        \mathbb{P}_Q(R_T(\pi, f_V) \geq aT\Delta_k/2 \vert X_{1:T} )
        \geq \frac{1}{2}.
    \end{align}
    Recall that \cref{eq:kl_one_eigths} gives $D_{\mathrm{KL}}(Q \Vert P) \leq \frac{1}{8}$.
    Then, applying \Cref{lem:change_of_measure} with $H = \mathbb{I}_{\{R_T(\pi, f_V) \geq aT\Delta_k/2\}}$ and $M = 1$ gives
    \begin{align}
        \mathbb{P}_P(R_T(\pi, f_V) \geq aT\Delta_k/2 )
        \geq \frac{1}{2} - \sqrt{\frac{1}{2}\cdot\frac{1}{8}}
        = \frac{1}{4}.
    \end{align}
    It remains to show that $aT\Delta_k/2 \geq \min \{1, 2\Delta_k \} \cdot \frac{\sigma}{32} \sqrt{Td_T }$.
    Since $d_T < 4$, we have $\sigma\sqrt{T}/2 \geq \frac{\sigma}{4}\sqrt{Td_T}$. Also, \Cref{lem:d_T_bounds} gives $d_T \leq \rho^{-1} = \frac{B^2T}{\sigma^2}$, such that $BT \geq \sigma\sqrt{Td_T}$. Combining gives $aT \geq \frac{\sigma}{4}\sqrt{Td_T}$, such that
    \begin{align}
        aT\Delta_k/2 
        \geq \frac{\sigma\Delta_k}{8}\sqrt{Td_T}
        \geq \min \{1, 2\Delta_k \} \cdot \frac{\sigma}{32} \sqrt{Td_T}.
    \end{align}
    Since $V$ is uniform on $\{-1,+1\}$, there exists some fixed $v \in \{-1,+1\} $ such that $R_T(\pi, f_v) \geq \min \{1, 2\Delta_k \} \cdot \frac{\sigma}{32} \sqrt{Td_T}$ with probability at least $1/4$.

    Finally, combining the two cases proves that for every policy $\pi$, there exists $f\in\mathcal{F}_k(B)$, such that with probability at least $1/4$,
    \begin{align}
        R_T(\pi, f)
        \geq \min \{1, 2\Delta_k \} \cdot \frac{\sigma}{32} \sqrt{Td_T},
    \end{align}
    which combined with \Cref{lem:d_T_bound_compact} proves the result.
\end{proof}

\section{Omitted Proofs from Section \ref{sec:discussion} }
\label{app:proofs_discussion}

\subsection{Deriving Sharp Rates for $\gamma_T$}
\label{app:derive_gamma_T_rates}

We recall that $\mathcal{X} = [0,1]^d$.
For kernels from three spectral regimes,
\citep{Janz26} presents sharp rates for the maximum information gain, which they define as
\begin{align}
    \Gamma^k_T(\mathcal{X}; \eta)
    := \sup_{x_{1:T} \in \mathcal{X}^T} \log \det (I_T + \eta^{-1} K_T(x_{1:T})), \label{eq:janz_maximum_information_gain}
\end{align}
where $1 \leq \eta \leq T$ is fixed. 
Recall from \cref{eq:gamma_T,eq:gamma_T_extended} that our maximum information gain is defined as $\gamma_T= \gamma_T(\mathcal{X}) =  \frac{1}{2} \max_{x_{1:T} \in \mathcal{X}^T} \log\det (I_T + \lambda^{-1} K_T(x_{1:T}) )$.
It follows that $\gamma_T = \frac{1}{2} \Gamma^k_T(\mathcal{X}; \lambda) $. It is possible that $0 < \lambda < 1$. In this case, we observe that $\log ( 1 + u) \leq \log (1 + u/\lambda) \leq \lambda^{-1} \log(1 + u) $ for every $u \geq 0$. 
Applying these inequalities to \cref{eq:janz_maximum_information_gain} gives $\Gamma^k_T(\mathcal{X}; 1) \leq \Gamma^k_T(\mathcal{X}; \lambda) \leq \lambda^{-1} \Gamma^k_T(\mathcal{X}; 1)$. Consequently, we can directly apply the asymptotic bounds of $\Gamma^k_T(\mathcal{X}; \eta)$ in \citep{Janz26}.

We now show how to get each entry in the second column of \Cref{tab:gamma_T_rates}.
For the Mat\'ern and squared exponential kernels, their rates directly come from \citep[Table 2]{Janz26}. 
For the $\gamma$-exponential and piecewise-polynomial kernels, \citep[Proposition 1]{Lee25} shows their Fourier transforms have polynomial decay, which combined with the second row of \citep[Table 1]{Janz26} gives their respective rates.
For the Dirichlet kernel, the upper bound is from \citep[Table 2]{Lee25}, and \Cref{lem:gamma_T_log_T_lower_bound} gives $\Omega(\log T)$.

For the rational quadratic kernel, some extra effort is needed, and we provide the following proposition. 

\begin{proposition}
    \label{prop:rq_gamma_T}
    Consider the setup of \Cref{sec:setup_compact} with $\mathcal{X} = [0,1]^d$. For the rational quadratic kernel, we have
    \begin{align}
        \gamma_T = \Theta((\log T)^{d+1}).
    \end{align}
\end{proposition}

\begin{proof}
    For $x,x'\in \mathcal{X}$, define $h := x-x'$. 
    Following \citep{Janz26}, we define $s(\xi)$ to be the spectral density related to the rational quadratic kernel, such that
    \begin{align}
        k_{\mathrm{RQ}, a}(h) = \int_{\mathbb{R}^d} \exp(2\pi i \xi\cdot h) s(\xi) \ d\xi. \label{eq:spectral_density}
    \end{align}
    
    \textbf{Step 1: Deriving the rate of $s(\xi)$.} This step is similar to \citep[pp.~21--22]{Lee25}, but extends to $d$ dimensions.
    For $u > 0$, the Gamma function satisfies $\int_{0}^{\infty} t^{a-1} \exp(-ut) \ dt = u^{-a} \Gamma(a)$. Then
    \begin{align}
        k_{\mathrm{RQ},a}(h) 
        = \biggl( 1 + \frac{\norm{h}_2^2}{2a\ell^2} \biggr)^{-a} 
        = \frac{1}{\Gamma(a)} \int_{0}^{\infty} t^{a-1}\exp(-t) \exp \biggl( - \frac{t \norm{h}_2^2}{2a\ell^2} \biggr) \ dt. \label{eq:rq_to_gaussian}
    \end{align}
    In $d$ dimensions, the Fourier transform of a Gaussian is again a Gaussian (e.g., see \cite[eq. (4.6) and discussion after eq. (4.9)]{Rasmussen05}):
    \begin{align}
        \int_{\mathbb{R}^d} \exp \biggl( - \frac{t \norm{h}_2^2}{2a\ell^2} \biggr) \exp(-2\pi i \xi \cdot h) \ dh
        = \biggl(\frac{2\pi a\ell^2}{t}\biggr)^{d/2} \exp\biggl(- \frac{2 \pi^2 a\ell^2 \norm{\xi}^2_2}{t} \biggr),
    \end{align}
    where the right-hand side is the spectral density of the Gaussian in \cref{eq:rq_to_gaussian}, which we denote as $q_t(\xi)$. Then
    \begin{align}
        \exp \biggl( - \frac{t \norm{h}_2^2}{2a\ell^2} \biggr) 
        = \int_{\mathbb{R}^d} \exp(2\pi i \xi \cdot h) q_t(\xi) \ d\xi. \label{eq:gaussian_spectral_density}
    \end{align}
    In particular, setting $h = 0$ gives $\int_{\mathbb{R}^d} q_t(\xi) \ d\xi = 1 $ and $q_t(\xi) \geq 0$.
    Substituting \cref{eq:gaussian_spectral_density} into \cref{eq:rq_to_gaussian} gives
    \begin{align}
        k_{\mathrm{RQ},a}(h) 
        = \frac{1}{\Gamma(a)} \int_{0}^{\infty} t^{a-1}\exp(-t) \biggl[ \int_{\mathbb{R}^d} \exp(2\pi i \xi \cdot h) q_t(\xi) \ d\xi \biggr] \ dt.
    \end{align}
    Since $\int_{\mathbb{R}^d} q_t(\xi) \ d\xi = 1 $, the absolute integral is finite:
    \begin{align}
         \frac{1}{\Gamma(a)} \int_{0}^{\infty} \int_{\mathbb{R}^d}  t^{a-1}\exp(-t) \abs{\exp(2\pi i \xi \cdot h)} q_t(\xi) \ d\xi \ dt
         = \frac{1}{\Gamma(a)} \int_{0}^{\infty} t^{a-1}\exp(-t) \ dt
         = 1.
    \end{align}
    Then we can use Fubini's theorem to exchange the integrals and get
    \begin{align}
        k_{\mathrm{RQ},a}(h) 
        =  \int_{\mathbb{R}^d} \exp(2\pi i \xi \cdot h) \biggl[\frac{1}{\Gamma(a)} \int_{0}^{\infty} t^{a-1}\exp(-t)  q_t(\xi) \ dt \biggr]  \ d\xi,
    \end{align}
    such that the spectral density of the rational quadratic kernel is
    \begin{align}
        s(\xi) = \frac{1}{\Gamma(a)} \int_{0}^{\infty} t^{a-1}\exp(-t)  q_t(\xi) \ dt.
    \end{align}
    Substituting $q_t(\xi)$ to $s(\xi)$ gives
    \begin{align}
        s(\xi) 
        & = \frac{(2\pi a \ell^2)^{d/2}}{\Gamma(a)} \int_{0}^{\infty} t^{a-1-d/2} \exp\biggl(- t- \frac{2 \pi^2 a\ell^2 \norm{\xi}^2_2}{t} \biggr) \ dt \label{eq:s_xi} \\
        & = 2 \cdot \frac{(2\pi a \ell^2)^{d/2}}{\Gamma(a)} (z/2)^{a-d/2} K_{a-d/2}(z),
    \end{align}
    where the second equality changes variable by $z := 2\pi \sqrt{2a} \ell \norm{\xi}_2 $ and applies \cite[\href{https://dlmf.nist.gov/10.32.E10}{(10.32.10)}]{DLMF} and \cite[\href{https://dlmf.nist.gov/10.27.E3}{(10.27.3)}]{DLMF}. Applying \cite[\href{https://dlmf.nist.gov/10.40.E2}{(10.40.2)}]{DLMF} and expanding only the first term with $k = 0$ (which further applies \cite[\href{https://dlmf.nist.gov/2.1.E13}{(2.1.13)}]{DLMF}, \cite[\href{https://dlmf.nist.gov/2.1.E14}{(2.1.14)}]{DLMF}, and \cite[\href{https://dlmf.nist.gov/10.17.i}{(10.17(i))}]{DLMF}) gives
    \begin{align}
        s(\xi) 
        & = 2 \cdot \frac{(2\pi a \ell^2)^{d/2}}{\Gamma(a)} (z/2)^{a-d/2} \sqrt{\pi/2z} \exp(-z) (1 + O(z^{-1})) \\
        & = C\cdot  \norm{\xi}_2^{a-(d+1)/2} \exp(-2\pi \sqrt{2a}\ell \norm{\xi}_2) (1 + O(\norm{\xi}_2^{-1})),
    \end{align}
    such that $s(\xi)  = \Theta \bigl( \norm{\xi}_2^{a-(d+1)/2} \exp(-2\pi \sqrt{2a}\ell \norm{\xi}_2) \bigr) $ as $\norm{\xi}_2 \to \infty$. We define $b := 2\pi \sqrt{2a}\ell$. Then
    \begin{align}
        s(\xi) = \Theta\bigl( \norm{\xi}_2^{a-(d+1)/2} \exp(-b \norm{\xi}_2)  \bigr). \label{eq:rq_spectral_density_rate}
    \end{align}
    
    \textbf{Step 2: Upper bound on $\gamma_T$.} 
    For sufficiently large $\norm{\xi}_2$, the polynomial factor in \cref{eq:rq_spectral_density_rate} can be absorbed into $\exp(b\norm{\xi}_2/2)$. 
    Then there exist $R, C_2 > 0$, such that $s(\xi) \leq C_2 \exp\bigl( -\frac{b}{2} \norm{\xi}_2 \bigr)$ when $\norm{\xi}_2 \geq R$. Hence,
    \begin{align}
        \int_{\mathbb{R}^d} \exp(b\norm{\xi}_2/4) s(\xi) \ d\xi 
        & = \int_{\norm{\xi}_2 < R} \exp(b\norm{\xi}_2/4) s(\xi) \ d\xi  + \int_{\norm{\xi}_2 \geq R} \exp(b\norm{\xi}_2/4) s(\xi) \ d\xi \\
        & \leq \exp(bR/4)  + C_2 \int_{\norm{\xi}_2 \geq R} \exp(-b\norm{\xi}_2/4) \ d\xi \\
        & \leq \exp(bR/4)  + C_2 \abs{\mathbb{S}^{d-1}} \int_{0}^{\infty} r^{d-1}\exp(-br/4) \ dr \\
        & = \exp(bR/4)  + C_2 \abs{\mathbb{S}^{d-1}} (4/b)^d \Gamma(d) \\
        & < \infty,
    \end{align}
    where the first inequality follows since $\int_{\mathbb{R}^d}s(\xi) \ d\xi = 1$, and the second inequality uses polar coordinates with $\abs{\mathbb{S}^{d-1}}$ being the unit sphere's surface area.
    Then, the conditions of \citep[Proposition 12]{Janz26} are satisfied with $\gamma = 1$ and $\tau_0 = b/4$, which gives $\epsilon_M = O \bigl( \exp(-cM^{\gamma}) \bigr)$. Then  the conditions of \citep[Corollary 3(2)]{Janz26} are satisfied with $\sigma = 1$, which gives $\gamma_T = O((\log T)^{d+1})$.

    \textbf{Step 3: Lower bound on $\gamma_T$.} 
    For sufficiently large $\norm{\xi}_2$, the polynomial factor in \cref{eq:rq_spectral_density_rate} can be lower bounded by $\exp(-b\norm{\xi}_2)$. 
    Then, there exist $R_2, C_3 > 0$ such that $s(\xi) \geq C_3 \exp\bigl( -2b \norm{\xi}_2 \bigr)$ when $\norm{\xi}_2 \geq R_2$. For $\norm{\xi}_2 < R_2$, restricting the integral in \cref{eq:s_xi} to $[1,2]$ gives
    \begin{align}
        s(\xi) 
        & \geq \frac{(2\pi a \ell^2)^{d/2}}{\Gamma(a)} \int_{1}^{2} t^{a-1-d/2} \exp\biggl(- t- \frac{2 \pi^2 a\ell^2 \norm{\xi}^2_2}{t} \biggr) \ dt \\
        & \geq \frac{(2\pi a \ell^2)^{d/2}}{\Gamma(a)} \exp\bigl(- 2- 2 \pi^2 a\ell^2 R_2^2 \bigr) \int_{1}^{2} t^{a-1-d/2} \ dt \\
        & \geq \biggl[\frac{(2\pi a \ell^2)^{d/2}}{\Gamma(a)} \exp\bigl(- 2- 2 \pi^2 a\ell^2 R_2^2 \bigr) \int_{1}^{2} t^{a-1-d/2} \ dt\biggr]  \cdot \exp(-2b\norm{\xi}_2),
    \end{align}
    where the last step follows because $\exp(-2b\norm{\xi}_2) \leq 1$, and the term in the brackets is positive.
    Then, there exists a constant $c_0$ such that $s(\xi) \geq c_0 \exp(-2b \norm{\xi}_2)$ for almost every $\xi \in \mathbb{R}^d$. Consequently, the conditions of \citep[Theorem 20]{Janz26} are satisfied with $\gamma = 1$ and $\tau = 2b$, which gives $\gamma_T = \Omega((\log T)^{d+1})$.
\end{proof}

\subsection{Proof of Lemma \ref{lem:lower_bound_diff_gamma} (Lower Bound Based on Information Difference)}

\begin{proof}
    Let $x_1,\ldots, x_T$ be a maximizing tuple of $\gamma_T(\mathcal{X})$, and let $\tilde{\mathcal{X}}$ be a finite approximation of $\mathcal{X}$ that contains this maximizing tuple (see \Cref{prop:finite_approximation}). Denote $\hat{\nu} := \frac{1}{T} \sum_{t=1}^T \delta_{x_t} \in \Delta(\tilde{\mathcal{X}}) $ as the empirical design.
    Recall that \Cref{thm:lower_bound_compact_X_d_T} gives
    \begin{align}
        R^*_T \geq c_1 \cdot \sqrt{Td_T(\tilde{\mathcal{X}})}, \label{eq:use_thm_2}
    \end{align}
    so it remains to lower bound $d_T(\tilde{\mathcal{X}})$ by $\gamma_T - \gamma_m$.
     
    To show this, we repeatedly delete entries from the maximizing tuple until $m$ points remain.
    For $t \in [T]$, we define $S_t := \sum_{s=1}^t \phi(x_s)\otimes \phi(x_s)$.
    When $t$ points remain, we delete the point that minimizes $\langle \phi(x), (S_t + \lambda I_E)^{-1} \phi(x) \rangle_k $, and denote the minimum as $\ell_t$.
    For convenience, we relabel the tuple such that $x_t$ is deleted when $t$ points remain.    
    Observe that
    \begin{align}
        0 \leq 
        \ell_t
        \leq \langle \phi(x_t), (\phi(x_t)\otimes \phi(x_t) + \lambda I_E)^{-1} \phi(x_t) \rangle_k
        = \frac{\norm{\phi(x_t)}_k^2}{\norm{\phi(x_t)}_k^2 + \lambda}
        \leq \frac{1}{1 + \lambda}
        < 1, \label{eq:ell_t_bound}
    \end{align}
    where the first inequality holds since $S_t + \lambda I_E \succeq 0$, the second inequality holds since $S_t \succeq \phi(x_t)\otimes \phi(x_t)$,
    the equality holds since $(\phi(x_t)\otimes \phi(x_t) + \lambda I_E)\phi(x_t) = (\norm{\phi(x_t)}_k^2 + \lambda)\phi(x_t)$, and the third inequality uses $\norm{\phi(x)}_k^2 = k(x,x) \leq 1$.

    We define $K_m(x_{1:m}) := [k(x_s, x_{s'})]_{s,s'\in[m]} $.
    \Cref{lem:same_eigenvalues} shows that $K_T(x_{1:T})$ and $S_T$ have the same positive eigenvalues (including multiplicity), and likewise $K_m(x_{1:m})$ and $S_m $, such that 
    \begin{align}
        \log \det_E (I_E + \lambda^{-1} S_T )
        = 2 \gamma_T(\mathcal{X}),
        \qquad
        \log\det_E (I_E + \lambda^{-1} S_{m} )  
        \leq 2 \gamma_{m}.
    \end{align}
    Then, we have
    \begin{align}
        2(\gamma_T - \gamma_{m})
        & \leq \log\det_E (I_E + \lambda^{-1} S_{T} ) - \log\det_E (I_E + \lambda^{-1} S_{m} ) \\
        & = \sum_{t = m+1}^T \biggl( \log\det_E (I_E + \lambda^{-1} S_{t} ) - \log\det_E (I_E + \lambda^{-1} S_{t-1} ) \biggr). \label{eq:telescope_sum}
    \end{align}
    Note that by definition, $S_{t} - \phi(x_t) \otimes \phi(x_t) = S_{t-1} $. 
    We now apply \Cref{lem:det_ratio} to $S_t + \lambda I_E$, which gives
    \begin{align}
        \frac{\det_E(S_{t-1} + \lambda I_E)}{\det_E(S_{t} + \lambda I_E)}
        = 1 - \langle \phi(x_t), (S_t+\lambda I_E)^{-1}\phi(x_t) \rangle_k
        = 1 - \ell_t, \label{eq:det_steps}
    \end{align}
    and hence
    \begin{align}
        \log \det_E(\lambda^{-1}  S_{t} + I_E) - \log \det_E(\lambda^{-1}  S_{t-1} + I_E)
        = - \log(1- \ell_t)
        \leq \frac{\ell_t}{1 - \ell_t}
        \leq (1 + \lambda^{-1}) \ell_t, \label{eq:log_det_diff}
    \end{align}
    where the first inequality follows since $-\log(1-\ell_t) = \int_{0}^{\ell_t} \frac{du}{1- u} \leq \frac{\ell_t}{1-\ell_t} $, and the last inequality uses \cref{eq:ell_t_bound}.
    Substituting \cref{eq:log_det_diff} back into \cref{eq:telescope_sum} gives
    \begin{align}
        2(\gamma_T - \gamma_{m})
        & \leq \sum_{t = m+1}^T (1 + \lambda^{-1}) \ell_t. \label{eq:gamma_diff_by_ell_t}
    \end{align}
    To upper bound $\ell_t$, we observe that
    \begin{align}
        \ell_t 
        & \leq \frac{1}{t} \sum_{s=1}^t \langle \phi(x_s), (S_t + \lambda I_E)^{-1} \phi(x_s) \rangle_k \label{eq:min_by_avg} \\
        & = \frac{1}{t} \tr_E(S_t (S_t + \lambda I_E)^{-1} ) \label{eq:use_lem_13} \\
        & \leq \frac{1}{t} \tr_E(S_T (S_T + \lambda I_E)^{-1} ) \label{eq:S_t_by_S_T} \\
        & = \frac{1}{t} \tr_E(M_{\hat{\nu}} (M_{\hat{\nu}} + \rho I_E)^{-1} ) \label{eq:S_T_to_m_nu} \\
        & \leq \frac{1}{t} 2d_T(\tilde{\mathcal{X}}), \label{eq:ell_t_to_d_T}
    \end{align}
    where \cref{eq:min_by_avg} follows since the minimum is upper bounded by the average, \cref{eq:use_lem_13}
    follows from $S_t \succeq 0$ and \Cref{lem:trace_to_inner_product}, \cref{eq:S_t_by_S_T} follows from $S_t \preceq S_T$, \cref{eq:S_T_to_m_nu} follows from $S_T = TM_{\hat{\nu}} $, and \cref{eq:ell_t_to_d_T} follows from \Cref{lem:bound_by_two_d_T}.
    Substituting \cref{eq:ell_t_to_d_T} into \cref{eq:gamma_diff_by_ell_t} gives $(\gamma_T - \gamma_{m}) \leq (1 + \lambda^{-1}) d_T(\tilde{\mathcal{X}}) \sum_{t = m+1}^T t^{-1} $, which combined with \cref{eq:use_thm_2} proves the result.
\end{proof}

\subsection{Proof of Proposition \ref{prop:sufficient_condition_to_remove_log} (Sufficient Condition for Log Factor Removal)}
\begin{proof}
    Choose $T$ sufficiently large that satisfies \cref{eq:growth_condition} and $T \geq 2L$. Then $\lfloor T/L \rfloor \geq T/(2L) $, such that $\sum_{t = \lfloor T/L \rfloor+1}^T \frac{1}{t} \leq \int_{\lfloor T/L \rfloor}^T u^{-1} \ du = \log(T/\lfloor T/L \rfloor) \leq \log (2L)  $. 
    Combining this with \cref{eq:growth_condition}, and then applying \Cref{lem:lower_bound_diff_gamma}, we obtain
    \begin{align}
        R^*_T 
        \geq  c_1 \cdot \sqrt{ \frac{cT\gamma_T }{(1+\lambda^{-1}) \log(2L) }}
        = \Omega(\sqrt{T\gamma_T}),
    \end{align}
    where $c_1 = \min \bigl\{ \frac{\sigma}{32}, \frac{\sigma \Delta_k}{16}  \bigr\}$, which can be absorbed into $\Omega(\cdot)$ because $\mathcal{X}$ is assumed to be fixed.
\end{proof}

\subsection{Proof of Lemma \ref{lem:connect_fV_to_gamma_T} (Upper Bound on $h(\mathcal{V},\rho)$)}

\begin{proof}
    Fix a nonempty set $\mathcal{V} \subseteq \mathcal{X}$. We consider greedily choosing $T$ points in $\mathcal{V}$ as follows: Define $S_0 := 0$, and at time $t \in [T]$, choose $x_t \in \mathcal{V}$ maximizing $\langle \phi(x), (S_{t-1} + \lambda I_E)^{-1} \phi(x) \rangle_k$. Denote this maximum by $a_t$, and set $S_{t} := S_{t-1} + \phi(x_t)\otimes \phi(x_t) $.
    We also define $a_{T+1}:= \max_{x\in\mathcal{V}} \langle \phi(x), (S_{T} + \lambda I_E)^{-1} \phi(x) \rangle_k $.
    For every $t \in [T]$, $S_t \succeq S_{t-1} $, such that $a_1 \geq \cdots \geq a_{T+1}$.
    Define $\hat{\nu} := \frac{1}{T}\sum_{t=1}^T \delta_{x_t} $ as the empirical design. Then $TM_{\hat{\nu}} = S_T$, such that
    \begin{align}
        h(\mathcal{V}, \rho)
        \leq \max_{x\in\mathcal{V}} \langle \phi(x), (M_{{\hat{\nu}}} + \rho I_E)^{-1} \phi(x) \rangle_k
        = T\max_{x\in\mathcal{V}} \langle \phi(x), (S_T + \lambda I_E)^{-1} \phi(x) \rangle_k
        = T a_{T+1}
        \leq \sum_{t=1}^Ta_t. \label{eq:f_by_sum_a_t}
    \end{align}
    To connect $h(\mathcal{V}, \rho)$ to $\gamma_T$,  we notice that
    \begin{align}
        2\gamma_T
        & \geq \log \det(I_T + \lambda^{-1} K_T(x_{1:T}) ) \\
        & = \log \det_E(I_E + \lambda^{-1} S_T ) \label{eq:use_lem_19} \\
        & = \sum_{t=1}^T \log \frac{\det_E(\lambda I_E + S_{t} )}{\det_E( \lambda I_E + S_{t-1} )} \label{eq:telescope}  \\
        & = \sum_{t=1}^T \log \bigl( 1 + \langle \phi(x_t), (\lambda I_E + S_{t-1})^{-1} \phi(x_t) \rangle_k \bigr) \label{eq:apply_det_ratio_2}  \\
        & = \sum_{t=1}^T \log (1+a_t), \label{eq:gamma_T_by_a_t}
    \end{align}
    where \cref{eq:use_lem_19} follows from \Cref{lem:same_eigenvalues}, \cref{eq:telescope} follows by telescoping, \cref{eq:apply_det_ratio_2} follows by applying \Cref{lem:det_ratio_2} with $A = \lambda I_E + S_{t-1}$, and \cref{eq:gamma_T_by_a_t} follows from the definition of $a_t$.
    For every $t \in [T]$, $(S_{t-1} + \lambda I_E )^{-1} \preceq \lambda^{-1} I_E $, 
    such that $a_t \leq \langle \phi(x_t),  \lambda^{-1} \phi(x_t) \rangle_k = \lambda^{-1}k(x_t,x_t) \leq \lambda^{-1} $, which further implies $\log(1+a_t) = \int_0^{a_t}\frac{du}{1+u} \geq \frac{a_t}{1 + \lambda^{-1}} $. Combining this result with \cref{eq:f_by_sum_a_t,eq:gamma_T_by_a_t} gives
    \begin{align}
        2\gamma_T
        \geq \sum_{t=1}^T \frac{a_t}{1 + \lambda^{-1}}
        \geq \frac{h(\mathcal{V}, \rho)}{1 + \lambda^{-1}},
    \end{align}
    and maximizing over $\mathcal{V}$ proves the result.
\end{proof}

\subsection{Proof of Theorem \ref{thm:regret_upper_bound_on_cube} (General Upper Bound)}
\begin{proof}
    We construct $\tilde{\mathcal{X}} := \{ 0, 1/m, \ldots, 1 \}^d $, where $m := \lceil \sqrt{d} (CT)^{1/\alpha} \rceil $.  By a ``rounding'' argument to the closest grid point, we see that for every $x \in \mathcal{X}$, there exists $a \in \tilde{\mathcal{X}}$ such that $\norm{x - a}_2 \leq \sqrt{d}/(2m) $.
    Then for every $f \in \mathcal{F}_{k}(B)$, we have
    \begin{align}
        \abs{f(x) - f(a)}
        \leq B \norm{\phi(x) - \phi(a)}_k
        \leq BC (\sqrt{d}/(2m))^{\alpha}
        \leq B / (T 2^{\alpha})
        \leq B/T,
    \end{align}
    where the first inequality follows from the reproducing property and $\norm{f}_k \leq B$, the second follows from \cref{eq:holder_condition}, the third follows since $m \geq \sqrt{d} (CT)^{1/\alpha} $, and the last follows because $\alpha > 0$. Hence, we have $\epsilon(\tilde{\mathcal{X}}) \leq B/T$ for the quantity $\epsilon(\tilde{\mathcal{X}})$ defined in \eqref{eq:epsilon_x}.
    
    We then apply the RIPS algorithm to $\tilde{\mathcal{X}}$, such that the minimax regret on $\mathcal{X}$ satisfies
    \begin{align}
        R^*_T 
        & = O \biggl( \log(\abs{\mathcal{\tilde{X}}}T) + \sqrt{T\gamma_T(\tilde{\mathcal{X}})\log(\abs{\tilde{\mathcal{X}}} T)} \biggr) + T\epsilon(\mathcal{\tilde{\mathcal{X}}}) \\
        & = O \bigl( \log T + \sqrt{T\gamma_T(\tilde{\mathcal{X}})\log T} + B \bigr) \\
        & = O \bigl(\sqrt{T\gamma_T(\tilde{\mathcal{X}})\log T} \bigr) \\
        & = O \bigl(\sqrt{T\gamma_T\log T} \bigr),
    \end{align}
    where the first step follows from \Cref{lem:cam_thm_3,lem:connect_fV_to_gamma_T} along with \cref{eq:epsilon_x},
    the second step follows because $\abs{\tilde{\mathcal{X}}} = (m+1)^d \leq ( \sqrt{d} (CT)^{1/\alpha} + 2)^d$ and $\epsilon(\tilde{\mathcal{X}}) \leq B/T$,
    the third step follows since $\gamma_T(\tilde{\mathcal{X}}) = \Omega(\log T)$ by \Cref{lem:gamma_T_log_T_lower_bound},
    and the last step follows since $\gamma_T(\tilde{\mathcal{X}}) \leq \gamma_T(\mathcal{X}) = \gamma_T$.
\end{proof}

\subsection{Proof that All Six Kernels Satisfy the H\"{o}lder Condition}
\label{app:kernel_satisfy_holder_condition}

Recall from Appendix \ref{app:kernel} that $k(\norm{x-x'}_2) = k(x,x')$.
For $x,y \in \mathcal{X}$, we define $r:= \norm{x-y}_2$, such that $\norm{\phi(x) - \phi(y)}_k^2 = k(x,x) + k(y,y) - 2k(x,y) = 2(k(0) - k(r)) $. Then, it suffices to prove that $k(0) - k(r) \leq \frac{C^2}{2}r^{2\alpha}$ for some fixed $C,\alpha > 0$.
We recall that $1 - \exp(-x) \leq x$ for every $x \in \mathbb{R}$. 

\textbf{Squared exponential kernel}. $k(0) - k(r) = 1 - \exp\bigl( -\frac{r^2}{2\ell^2} \bigr) \leq \frac{r^2}{2\ell^2}$. Then, we can choose $\alpha = 1$.

\textbf{Rational quadratic kernel}. Since $\log(1 + x) \leq x$ for every $x \geq 0$, we have
\begin{align}
    k(0) - k(r) 
    = 1 - \biggl( 1 + \frac{r^2}{2a\ell^2} \biggr)^{-a}
    = 1 - \exp \biggl( - a\log \Bigl(1 + \frac{r^2}{2a\ell^2} \Bigr) \biggr)
    \leq a\log \biggl(1 + \frac{r^2}{2a\ell^2} \biggr)
    \leq \frac{r^2}{2\ell^2},
\end{align}
such that we can choose $\alpha = 1$.

\textbf{$\gamma$-exponential kernel}. $k(0) - k(r) = 1 - \exp(-(r/\ell)^\gamma) \leq (r/\ell)^{\gamma}$. Hence, we can choose $\alpha = \gamma/2$.

\textbf{Mat\'ern kernel}. 
For $z \geq 0$, $K_{\nu}(z)$ can be represented by \cite[\href{https://dlmf.nist.gov/10.32.E10}{(10.32.10)}]{DLMF} as follows:
\begin{align}
    K_\nu(z) 
    = \frac{1}{2} \biggl(\frac{z}{2}\biggr)^{\nu} \int_0^{\infty} \exp\biggl(-t-\frac{z^2}{4t}\biggr)\frac{dt}{t^{\nu+1}}.
\end{align}
We change variable by $t = z^2/(4s)$, which gives
\begin{align}
    K_\nu(z) 
    = z^{-\nu} 2^{\nu-1} \int_0^{\infty} s^{\nu-1} \exp\biggl(- \frac{z^2}{4s} -s\biggr) \ ds,
\end{align}
such that
\begin{align}
    \frac{2^{1-\nu}}{\Gamma(\nu)} z^{\nu} K_{\nu}(z) = \frac{1}{\Gamma(\nu)} \int_0^{\infty} s^{\nu-1} \exp\biggl(- \frac{z^2}{4s} -s\biggr) \ ds.
\end{align}
Taking $z = \sqrt{2\nu} r / \ell$ gives
\begin{align}
    k_\nu(r)
    = \frac{1}{\Gamma(\nu)} \int_0^{\infty} s^{\nu-1} \exp (-s) \exp\biggl(- \frac{\nu r^2}{2\ell^2s}\biggr) \ ds. \label{eq:mat_integral}
\end{align}
The integrand is upper bounded by $s^{\nu-1} \exp(-s)$, whose integral is $\Gamma(\nu) < \infty$. 
For $u \geq 0$ and $0 < \alpha \leq 1$, we have $1 - \exp(-u)\leq \min\{u, 1\} \leq u^{\alpha} $.
Then,
\begin{align}
    k(0) - k(r) 
    & = \frac{1}{\Gamma(\nu)} \int_0^{\infty} s^{\nu-1} \exp(-s) \biggl[1 - \exp\biggl(- \frac{\nu r^2}{2\ell^2s}\biggr)\biggr] \ ds \\
    & \leq \frac{1}{\Gamma(\nu)} \biggl(\frac{\nu r^2}{2\ell^2}\biggr)^{\alpha} \int_0^{\infty} s^{\nu-\alpha-1} \exp(-s) \ ds \\
    & = \frac{\Gamma(\nu-\alpha)}{\Gamma(\nu)} \biggl(\frac{\nu}{2\ell^2}\biggr)^{\alpha} r^{2\alpha},
\end{align}
where the ratio of Gamma functions is finite provided $\nu - \alpha > 0$. 
Then we can choose any $0 < \alpha \leq 1$ with $\alpha < \nu$.

\textbf{Piecewise-polynomial kernel}.
If $0 \leq r \leq 1$, then 
\begin{align}
    k(0) - k(r)
    = - \sum_{j=1}^{\lfloor d/2 \rfloor + 3q + 1} c_{j,q} r^j
    \leq \sum_{j=1}^{\lfloor d/2 \rfloor + 3q + 1}\abs{c_{j,q}} r^j
    \leq \sum_{j=1}^{\lfloor d/2 \rfloor + 3q + 1}\abs{c_{j,q}} r.
\end{align}
If $r > 1$, then $k(0) - k(r) = c_{0,q} \leq \abs{c_{0,q}} r $. Then $k(0) - k(r) \leq \max\bigl\{\sum_{j=1}^{\lfloor d/2 \rfloor + 3q + 1}\abs{c_{j,q}}, \abs{c_{0,q}} \bigr \} \cdot r  $, and we can choose $\alpha = 1/2$.

For the one-dimensional Dirichlet kernel, pairing terms with index $j$ and $-j$ gives
\begin{align}
    k_{\mathrm{Dir},n}(r)
    = \frac{1}{2n+1} \sum_{k=-n}^n \exp(-ikr)
    = \frac{1}{2n+1} \biggl(1 + 2 \sum_{j=1}^n \cos(jr) \biggr)
\end{align}
Specifically, $k(0) = 1$. 
Then
\begin{align}
    k(0) - k(r)
    = 1 - \frac{1}{2n+1} \biggl(1 + 2 \sum_{j=1}^n \cos(jr) \biggr)
    = \frac{2}{2n+1} \sum_{j=1}^n (1 - \cos(jr))
    \leq \frac{r^2}{2n+1} \sum_{j=1}^n j^2
    = \frac{n(n+1)}{6} r^2,
\end{align}
where the inequality follows because $1 - \cos u = 2\sin^2 (u/2) \leq u^2/2 $, and the last equality follows because $\sum_{j=1}^n j^2 = n(n+1)(2n+1)/6 $. Then we can choose $\alpha = 1$.

\subsection{Proof of Lemma \ref{lem:f_V_sharper} (Sharper Upper Bounds; First Result)}

\begin{proof}
    Fix a nonempty $\mathcal{V} \subseteq \mathcal{X}$, and denote $m := \abs{\mathcal{V}} $. 
    Following \cref{eq:d_optimal_design}, we define $\nu$ as a regularized $D$-optimal design over $\Delta(\mathcal{V})$. 
    We also define $E_1 :=\mathrm{span} \{\phi(v): v \in \mathcal{V} \} $.
    Define $D := \diag(\nu(v), v \in \mathcal{V} ) $ and the weighted Gram matrix $H_k := D^{1/2} [k(v,w)]_{v,w\in\mathcal{V}} D^{1/2} $. 
    By an argument similar to that in \Cref{lem:same_eigenvalues}, we can show that $H_k$ and $M_{\nu}$ have the same nonzero eigenvalues (including multiplicities) as follows: Define $\Phi: \mathbb{R}^m \to E_1 $ by $\Phi b_i = \sqrt{\nu(v_i)} \phi(v_i) $, and define $\Phi^*$ to be its adjoint. Then $H_k = \Phi^*\Phi $ and $M_{\nu} = \Phi \Phi^* $, and the result follows.
    
    Then, \citep[Lemma 3]{Camilleri21} gives $h(\mathcal{V}, \rho) \leq \tr_{E_1}(M_{\nu} (M_{\nu} + \rho I_{E_1})^{-1})$, which combined with the above fact on $H_k, M_{\nu}$ implies
    \begin{align}
        h(\mathcal{V}, \rho) \leq \tr(H_k (H_k + \rho I)^{-1}). \label{eq:f_V_by_H_k}
    \end{align}
    To proceed, we apply \citep[Theorem 1]{Janz26}.
    Suppose that for every $M \in \mathbb{Z}_{+}$, there exist
    positive semidefinite kernels $k_{M,1}, k_{M,2}: \mathcal{X}\times \mathcal{X} \to \mathbb{R}$ and $\epsilon_M > 0$ such that
    \begin{align}
        k \preceq k_{M,1} + k_{M,2},
        \qquad
        \mathrm{rank}[k_{M,1}(x,y)]_{x,y\in\mathcal{X}} =  O(M^d),
        \qquad
        \max_{x \in\mathcal{X}} k_{M,2}(x,x) \leq \epsilon_M^2, \label{eq:kernel_condition}
    \end{align}
    where $k \preceq k_{M,1} + k_{M,2}$ denotes the corresponding Gram-matrix inequality. Following $H_k$, we define $H_{k_{M,1}}$ and $H_{k_{M,2}}$ as the weighted Gram matrices for $k_{M,1}$ and $k_{M,2}$, respectively. Then
    \begin{align}
        H_k \preceq H_{k_{M,1}} + H_{k_{M,2}},
        \qquad
        \mathrm{rank} \ H_{k_{M,1}} =  O(M^d),
        \qquad
        \tr H_{k_{M,2}} = \sum_{v\in\mathcal{V}} \nu(v)k_{M,2}(v,v) \leq \epsilon_M^2. \label{eq:H_facts}
    \end{align}
    Then, we follow the proof of \citep[Theorem 1]{Janz26} to upper bound $\tr(H_k (H_k + \rho I)^{-1})$.
    For convenience, we define $H := H_{k_{M,1}} + H_{k_{M,2}} $.
    Since $H_k \preceq H$, it follows that $\tr(H_k (H_k + \rho I)^{-1}) \leq \tr (H (H + \rho I)^{-1}) $. Let $P$ be the orthogonal projection onto the range of $H_{k_{M,1}}$, and define $Q := I - P$. Then
    \begin{align}
        \tr(H_k (H_k + \rho I)^{-1}) 
        & \leq \tr (H (H + \rho I)^{-1}) \\
        & = \tr (PH (H + \rho I)^{-1}P) + \tr (QH (H + \rho I)^{-1}Q) \\
        & \leq \mathrm{ rank } \ H_{k_{M,1}} + \tr (QH (H + \rho I)^{-1}Q) \\
        & \leq \mathrm{ rank } \ H_{k_{M,1}} + \rho^{-1} \tr (QH_{k_{M,2}} Q), \label{eq:trace_h_k}
    \end{align}
    where the second inequality uses $H (H + \rho I)^{-1} \preceq I$ and $\tr P = \mathrm{ rank } \ H_{k_{M,1}} $,
    the third inequality follows since $H(H+\rho I)^{-1} \preceq \rho^{-1}H $ and $Q H_{k_{M,1}} Q  = 0$. Applying \cref{eq:f_V_by_H_k,eq:H_facts} gives
    \begin{align}
        h(\mathcal{V}, \rho)
        = O ( M^d + \epsilon_M^2 / \rho )
        = O ( M^d + T\epsilon_M^2 ). \label{eq:f_V_by_r_eta}
    \end{align}

    We now specialize \cref{eq:f_V_by_r_eta} to each kernel.

    \textbf{Squared exponential kernel.}
    \citep[Remark 2]{Janz26} shows the squared exponential kernel has spectral density $s(\xi) = \Theta \bigl( \exp(-\tau \norm{\xi}^2_2) \bigr) $ for some $\tau > 0$. Denote $\mu$ as its spectral measure. Then
    \begin{align}
        \int_{\mathbb{R}^d} \exp\biggl( \frac{\tau}{2} \norm{\xi}_2^2 \biggr) \mu(d\xi)
        = \int_{\mathbb{R}^d} \exp\biggl(  \frac{\tau}{2} \norm{\xi}_2^2 \biggr) s(\xi) \ d\xi
        \leq C \int_{\mathbb{R}^d} \exp\biggl( - \frac{\tau}{2} \norm{\xi}_2^2 \biggr) \ d\xi
        = C (2\pi / \tau)^{d/2} < \infty,
    \end{align}
    where the second equality follows by applying $\int_{\mathbb{R}} \exp(-\tau x^2/2) \ dx = \sqrt{2\pi/\tau} $ to each dimension.
    Then the conditions of \citep[Proposition 15]{Janz26} are satisfied with $\tau_0 = \tau / 2$ and $\gamma = 2$, which implies $k_{\mathrm{SE}}$ (defined on $[-1,1]^d\times [-1,1]^d $) is entire of finite order at most $p = 2$.
    Then \citep[Proposition 14]{Janz26} gives $\epsilon_M = O\bigl(\exp (-cM\log(e+M))\bigr) $ for some $c > 0$, and restricting to $\mathcal{X}\times \mathcal{X}$ preserves this bound.
    We choose $M = \Bigl \lceil \frac{2}{c} \frac{\log T}{\log\log T} \Bigr \rceil $. For sufficiently large $T$, $\log(e + M) \geq \frac{1}{2}\log\log T$, such that $\epsilon_M^2 = O (T^{-2})$, and
    \begin{align}
        h(\mathcal{V}, \rho)
        = O \bigl( (\log T / \log\log T)^d + T^{-1} \bigr)
        = O \bigl( (\log T / \log\log T)^d \bigr).
    \end{align}
    
    \textbf{Rational quadratic kernel}. We define $\tilde{k}_{\mathrm{RQ}}$ as the extension of $k_{\mathrm{RQ}}$ to $\mathbb{R}^d \times \mathbb{R}^d$. Then $\tilde{k}_{\mathrm{RQ}}$ is positive semidefinite, because it can be represented as a positive mixture of squared exponential kernels \citep[(4.19)--(4.20)]{Rasmussen05}. In addition, $\tilde{k}_{\mathrm{RQ}}$ is real analytic on $\mathbb{R}^d \times \mathbb{R}^d$, such that $\tilde{k}_{\mathrm{RQ}}$ belongs to the Gevrey class $G^1(\mathbb{R}^d \times \mathbb{R}^d)$, which is the class of real-analytic functions on $\mathbb{R}^d \times \mathbb{R}^d$ \citep{Janz26}.  Then \citep[Corollary 11]{Janz26} is satisfied with $U = \mathbb{R}^d$, $\sigma = 1$, $\tilde{k} =  \tilde{k}_{\mathrm{RQ}}$, such that $\tilde{k}_{\mathrm{RQ}} $ (restricted to $[0,1]^d\times [0,1]^d$) has $\epsilon_M = O (\exp(-cM)) $, and further restricting to $\mathcal{X}\times \mathcal{X}$ preserves this bound.
    Then, choosing $M = \lceil c^{-1} \log T \rceil$ gives
    \begin{align}
        h(\mathcal{V}, \rho)
        = O \bigl( (\log T)^d + T^{-1} \bigr)
        = O \bigl( (\log T)^d \bigr),
    \end{align}
    and hence $\max_{\mathcal{V} \subseteq \mathcal{X}}h(\mathcal{V}, \rho) = O((\log T)^d)$.
    
    \textbf{One-dimensional Dirichlet kernel.}
    Observe that
    \begin{align}
        k_{\mathrm{Dir},n}(x,x')
        & = \frac{1}{2n+1} \sum_{k=-n}^n \exp(-ik(x-x')) \\
        & = \frac{1}{2n+1} \biggl(1 +  2\sum_{j=1}^n \cos(j(x-x')) \biggr) \\
        & = \frac{1}{2n+1} \biggl(1 +  2\sum_{j=1}^n [\cos(jx)\cos(jx') + \sin(jx)\sin(jx')] \biggr), \label{eq:kernel_decompose}
    \end{align}
    where each of the $2n+1$ terms in \cref{eq:kernel_decompose} is a positive semidefinite kernel with rank at most $1$. 
    Then, $k_{\mathrm{Dir},n}(r)$ has rank at most $2n+1$. Choosing $k_{M,1} = k$ and $k_{M,2} \equiv 0$ gives $\mathrm{rank} \ H_{k_{M,1}} \leq 2n + 1$ and $\tr H_{k_{M,2}} = 0$, which combined with
    \cref{eq:f_V_by_H_k,eq:trace_h_k} gives
    \begin{align}
        h(\mathcal{V}, \rho)
        \leq 2n + 1,
    \end{align}
    and hence $\max_{\mathcal{V} \subseteq \mathcal{X}}h(\mathcal{V}, \rho) \le 2n+1$.
\end{proof}

\subsection{Proof of Lemma \ref{lem:remove_log_MOSS} (Sharper Upper Bounds; Second Result)}

For $\mathcal{X} = [0,1]^d$, consider the finite approximation $\tilde{\mathcal{X}} = \{0,1/m, \ldots, 1\}^d $, where $m$ will be chosen later. 
We overload notation and write $k(x,y) = k(x-y)$.
For each of the kernels, we consider its spectral representation
\begin{align}
    k(z) = \int_{\mathbb{R}^d} \exp(iw^T z) p(w) \ dw, \label{eq:spectral_representation}
\end{align}
with Fourier normalization constants absorbed into $p$. Specifically, $p(w) = (2\pi) ^{-d} s(w/2\pi) $, where $s(w)$ is the spectral density in \citep{Janz26}, and $p(w) = (2\pi)^{-d} \hat{k}(w) $, where $\hat{k}(w)$ is the Fourier transform of the kernel $k$ used in \citep{Lee25}.
Our proof relies on bounding $p(w)$, as shown below.

\begin{lemma}
    \label{lem:nice_discretization}
    Consider the setup of \Cref{sec:setup_compact} with $\mathcal{X} = [0,1]^d$.
    Suppose that there exist constants $A > 0$ and $0 < s < 2$, such that
    \begin{align}
        0 \leq p(w) \leq A (1 + \norm{w}_2^2)^{-(s + d/2)}
        \qquad
        \forall w \in \mathbb{R}^d. \label{eq:spectral_assumption}
    \end{align}
    Then there exists a constant $C > 0$ (which may depend on $A,d,s$) such that
    \begin{align}
        \epsilon(\tilde{\mathcal{X}}) \leq C B m^{-s}.
    \end{align}
\end{lemma}

\begin{proof}
    Fix $x \in \mathcal{X}$, and consider a grid cell in $\tilde{\mathcal{X}}$ containing $x$.
    Then there exists $j\in \{0, \ldots, m-1 \}^d $, such that the vertices of the cell can be written as 
    \begin{align}
        v_b = m^{-1} (j + b) \in \tilde{\mathcal{X}}
        \qquad
        \forall b \in \{0, 1\}^d.
    \end{align}
    Likewise, there exists $\theta \in [0,1]^d$, such that $x = m^{-1}(j + \theta) $.
    For $\ell \in [d]$, we define $b_\ell$ to be the $\ell$-th coordinate of $b$, and similarly $\theta_{\ell}$.
    Then, we define the multilinear interpolation weights
    \begin{align}
        w_b(x) := \prod_{\ell:b_{\ell} = 1} \theta_{\ell} \prod_{\ell: b_{\ell} = 0} (1 - \theta_{\ell}). 
    \end{align}
    It follows that
    \begin{align}
        w_b(x) \geq 0,
        \qquad
        \sum_{b \in \{0,1\}^d } w_b(x) = 1,
        \qquad
        \sum_{b \in \{0,1\}^d } w_b(x) v_b = x,
        \qquad
        \norm{v_b - x}_2 \leq \sqrt{d}/m. \label{eq:weight_facts}
    \end{align}
    Define $I_m f(x) := \sum_{b \in \{0,1\}^d } w_b(x) f(v_b) $. 
    Because $w_b(x)$ are weights that sum to one, 
    we have $I_m f(x) \leq \max_{a\in\tilde{\mathcal{X}}} f(a) $.
    Consequently, $ f(x) - \max_{a \in \tilde{\mathcal{X}}}f(a) \leq f(x) - I_m f(x) \leq \abs{f(x) - I_m f(x)}  $, and taking the maximum over $x \in \mathcal{X}$ gives
    \begin{align}
        \max_{x \in \mathcal{X}} f(x) - \max_{a \in \tilde{\mathcal{X}}} f(a)
        \leq \sup_{x\in \mathcal{X} } \abs{f(x) - I_m f(x)}. \label{eq:approximation_error}
    \end{align}
    We proceed to upper bound ${f(x) - I_m f(x)}$.
    By the reproducing property,
    \begin{align}
        f(x) - I_m f(x)
        = \biggl \langle f, \phi(x) - \sum_{b \in \{0,1\}^d } w_b(x) \phi(v_b) \biggr \rangle_k.
    \end{align}
    Applying the Cauchy-Schwarz inequality and $\norm{f}_k \leq B$ gives
    \begin{align}
        \abs{f(x) - I_mf(x)}^2
        \leq B^2  \biggl\lVert \phi(x) - \sum_{b \in \{0,1\}^d } w_b(x) \phi(v_b) \biggr\rVert_k^2. \label{eq:Imf_bound}
    \end{align}
    For the squared norm term in \cref{eq:Imf_bound}, we have
    \begin{align}
        & \biggl\lVert \phi(x) - \sum_{b \in \{0,1\}^d } w_b(x) \phi(v_b) \biggr\rVert_k^2 \notag \\
        &= k(x,x) + \sum_{b,c} w_b(x) w_c(x) k(v_b, v_c) - \sum_b w_b(x) (k(x,v_b) +k(v_b, x)) \\
        &=\int_{\mathbb{R}^d} \biggl[ 1 + \sum_{b,c} w_b(x)w_c(x) \exp(iw^T(v_b - v_c)) - \sum_{b} w_b(x) \exp(iw^T(v_b - x)) - \sum_{b} w_b(x)\exp(iw^T(x-v_b))  \biggr] p(w) \ dw \\
        &= \int_{\mathbb{R}^d} \biggl( 1 - \sum_b w_b(x) \exp(iw^T (v_b - x)) \biggr) \biggl( 1 - \sum_c w_c(x) \exp(-iw^T (v_c - x)) \biggr) p(w) \ dw \\
        &= \int_{\mathbb{R}^d} \biggl\lvert 1 - \sum_b w_b(x) \exp(iw^T (v_b - x)) \biggr\rvert^2 p(w) \ dw, \label{eq:norm_bound}
    \end{align}
    where the second equality applies \cref{eq:spectral_representation},
    and the last equality follows because the weights are real and $\overline{e^{it}} = e^{-it}$.
    Now define $D_x(w) := 1 - \sum_{b} w_b (x) \exp(i w^{T}(v_b -x))$. Since $\abs{\exp(ui)} = \abs{\cos u + i \sin u} =  1$ for real $u$, we have $\abs{D_x(w)} \leq 1 + \sum_b \abs{w_b(x)} \abs{\exp(iw^T(v_b - x))} = 1 + \sum_b w_b(x) = 2 $. For a sharper bound, we first show
    \begin{align}
        \abs{\exp(iu) - 1 - ui} \leq u^2 / 2. \label{eq:integral_bound}
    \end{align}
    Indeed, $\exp(ui) - 1 - ui = ui \int_0^1 \exp(iut) - 1 \ dt = -u^2 \int_{0}^1\int_{0}^t \exp(isu) \ dsdt $. Since $\abs{\exp(isu)} = 1$, we have $\abs{\exp(ui) - 1 - ui} \leq u^2 \int_{0}^1\int_{0}^t 1 \ dsdt = u^2/2 $.
    Then
    \begin{align}
        \abs{D_x(w)}
        & = \biggl\lvert \sum_b w_b(x)(\exp(iw^T(v_b -x) - 1 - iw^T(v_b-x))  \biggr\rvert \\
        & \leq \frac{1}{2} \sum_{b} w_b(x) (w^T (v_b - x))^2 \\
        & \leq \frac{1}{2} \sum_{b} w_b(x) \norm{w}_2^2 \norm{v_b - x}_2^2 \\
        & \leq \frac{d}{2m^2} \norm{w}_2^2,
    \end{align}
    where the first equality uses $\sum_{b}w_b(x)w^T(v_b - x) = 0 $ from \cref{eq:weight_facts},
    the first inequality applies \cref{eq:integral_bound},
    the second inequality applies Cauchy-Schwarz,
    and the last inequality uses $\norm{v_b - x}_2 \leq \sqrt{d}/m$ from \cref{eq:weight_facts}.
    Then
    \begin{align}
        \abs{D_x(w)}^2 \leq C_d\cdot \min \bigl\{1, m^{-4}  \norm{w}^4_2  \bigr\}, \label{eq:dxw_bound}
    \end{align}
    where $C_d > 0 $ is a constant that depends on $d$.
    Combining
    \cref{eq:norm_bound,eq:Imf_bound} gives
    \begin{align}
        \abs{f(x) - I_mf(x)}^2
        & \leq B^2  \int_{\mathbb{R}^d} \abs{D_x(w)}^2 p(w) \ dw \\
        & \leq AB^2C_d  \int_{\mathbb{R}^d} \min \{1, m^{-4} \norm{w}^4_2  \} (1 + \norm{w}_2^2)^{-(s + d/2)} \ dw \\
        & \leq AB^2C_d \abs{\mathbb{S}^{d-1}} \int_0^{\infty} \min \{1, m^{-4} r^4  \} (1 + r^2)^{-(s + d/2)}r^{d-1} \ dr,
    \end{align}
    where the second inequality uses \cref{eq:spectral_assumption} and \cref{eq:dxw_bound}, and the third inequality uses the polar-coordinate with $\abs{\mathbb{S}^{d-1}}$ being the unit sphere's surface area. Now we decompose the integral according to the value of $r$.
    When $0 \leq r \leq 1$, $m^{-4} r^4 \leq 1$ and $(1+r^2)^{-(s+d/2)} \leq 1 $. Then the integral on $(0,1)$ is upper bounded by $\int_{0}^1 m^{-4} r^{d+3} dr = \frac{m^{-4}}{d+4} \leq m^{-4} $. For $1\leq r \leq m$, $(1+r^2)^{-(s+d/2)} \leq r^{2\cdot-(s+d/2)} $, and the integrand is at most $m^{-4} r^{3-2s} $. For $r \geq m$, the integral is at most $(1+r^2)^{-(s+d/2)} r^{d-1}\leq r^{-2s-d+d-1} = r^{-2s-1} $. Thus
    \begin{align}
        \abs{f(x) - I_mf(x)}^2
        & \leq AB^2C_d\abs{\mathbb{S}^{d-1}} \cdot \biggl[ m^{-4} + m^{-4}\int_{1}^{m} r^{3-2s} \ dr + \int_{m}^{\infty} r^{-2s-1} \biggr]. \label{eq:sum_integral}
    \end{align}

    For $0 < s < 2$, $m^{-4} \int_{1}^{m} r^{3 - 2s} \ dr = \frac{m^{-2s} - m^{-4}}{4 - 2s} \leq \frac{m^{-2s}}{4 - 2s}  $, and $\int_{m}^{\infty} r^{-2s-1} = \frac{m^{-2s}}{2s} $. Also $m^{-4} \leq m^{-2s}$. Hence,
    \begin{align}
        \abs{f(x) - I_mf(x)}^2
        \leq CB^2 m^{-2s}.
    \end{align}
    Taking the square root and combining with \cref{eq:approximation_error} proves the result.
\end{proof}

For a kernel satisfying \cref{eq:spectral_assumption}, applying MOSS on $\tilde{\mathcal{X}}$ gives
\begin{align}
    R^*_T 
    & \leq 39 \sigma\sqrt{\abs{\tilde{\mathcal{X}}}T} + 2B\abs{\tilde{\mathcal{X}}} + T \epsilon(\tilde{\mathcal{X}}) \\
    & \leq 39 \sigma\sqrt{T(m+1)^d} + 2B(m+1)^d + T CBm^{-s}.
\end{align}
Choosing $m = \lfloor T^{1/(2s+d)} \rfloor $ gives $R^*_T = O(T^{(s+d)/(2s+d)} )$.
We now show that each kernel satisfies \cref{eq:spectral_assumption}.

\textbf{Mat\'ern kernel.} \citep[Corollary 7]{Janz26} gives $p_\nu(w) = \Theta \bigl( (1 + \norm{w}^2_2)^{-(\nu + d/2)} \bigr) $, which satisfies \cref{eq:spectral_assumption} with $s = \nu$. $s \in (0,2)$ implies $\nu \in (0,2)$, in which case combining $R^*_T = O(T^{(s+d)/(2s+d)} )$ and $\gamma_T = \Theta(T^{d/(2\nu+d)})$ in \Cref{tab:gamma_T_rates} gives $R^*_T = O(T^{(\nu+d)/(2\nu+d)})= O(\sqrt{T\gamma_T})$.

\textbf{$\gamma$-exponential kernel.}
\citep[Proposition 1]{Lee25} gives $p(w) = \Theta \bigl( (1 + \norm{w}_2)^{-(\gamma + d)}\bigr) $. This implies $p(w) = O \bigl((1 + \norm{w}_2^2)^{-(\gamma/2 + d/2)}\bigr) $, such that \cref{eq:spectral_assumption} is satisfied with $s = \gamma/2$.\footnote{\citep[Proposition 1]{Lee25} gives asymptotic bounds. Since $p(\cdot)$ is bounded, we can enlarge $A$ in \cref{eq:spectral_assumption} such that \cref{eq:spectral_assumption} holds for every $w \in \mathbb{R}^d$. This procedure is similarly applied to the piecewise-polynomial kernel.}
Specifically, $\gamma \in (0,2)$ is covered by $s\in (0,2)$. Then, combining $R^*_T = O(T^{(s+d)/(2s+d)} )$ and $\gamma_T = \Theta(T^{d/(\gamma+d)})$ in \Cref{tab:gamma_T_rates} gives $R^*_T = O\bigl(T^{(\gamma/2+d)/(\gamma+d)}\bigr) = O(\sqrt{T\gamma_T}) $.

\textbf{Piecewise-polynomial kernel.}
\citep[Proposition 1]{Lee25} gives $p(w) = \Theta \bigl( (1 + \norm{w}_2)^{-(2q + 1 +d)} \bigr) $, such that $p(w) = O\bigl( (1 + \norm{w}_2^2)^{-(q + 1/2 + d/2)} \bigr) $. Then \cref{eq:spectral_assumption} is satisfied with
$s = q + \frac{1}{2}$, and $s < 2$ implies $q = 0, 1$, which combined with Appendix \ref{app:kernel} means $q= 0, d\geq 3$ or $q = 1$. In this case, combining $R^*_T = O(T^{(s+d)/(2s+d)} )$ and $\gamma_T = \Theta(T^{d/(2q + 1+d)})$ in \Cref{tab:gamma_T_rates} gives $R^*_T =  O\bigl(T^{(q + \frac{1}{2}+d)/(2q + 1 +d)}\bigr) = O(\sqrt{T\gamma_T})$.

\section{Upper Bound on the Maximum Information Gain}
\label{app:upper_bound_on_gamma_T}

Another consequence of \Cref{cor:lower_bound_copmact_X_gamma_T} is that an upper bound on $R^*_T$ implies an upper bound on $\gamma_T$. 
Here we use $R^*_t$ to denote minimax regret at horizon $t \in \mathbb{Z}_{+}$, and the following result considers $t \leq T$.

\begin{corollary}
    \label{cor:gamma_T_upper_bound}
    Consider the setup of \Cref{sec:setup_compact} where $\mathcal{X}$ is compact. 
    Suppose there exists a sequence $(U_t)_{t\in \mathbb{Z}_{+}}$ such that $R^*_t \leq U_t$ for every $t \in \mathbb{Z}_{+}$.
    Then,
    \begin{align}
        \gamma_T
        \leq \gamma_1 + \frac{1 + \lambda^{-1}}{c_1^2} \sum_{t=2}^T \frac{U_t^2}{t^2},
    \end{align}
    where $c_1 = \min \bigl\{ \frac{\sigma}{32}, \frac{\sigma \Delta_k}{16}  \bigr\}$. 
\end{corollary}

\begin{proof}
    If $T =1$, then the result holds with equality. Then, we assume $T \geq 2$.
    Applying \Cref{lem:lower_bound_diff_gamma} with $T = t$ and $m = t-1$ gives
    \begin{align}
        R^*_t \geq  c_1 \cdot t\sqrt{ \frac{(\gamma_t - \gamma_{t-1}) }{(1+\lambda^{-1})}},
    \end{align}
    which implies
    \begin{align}
        \gamma_t - \gamma_{t-1}
        \leq \frac{1 + \lambda^{-1}}{c_1^2 } \cdot \frac{(R^*_t)^2}{t^2}
        \leq \frac{1 + \lambda^{-1}}{c_1^2 } \cdot \frac{U_t^2}{t^2}.
    \end{align}
    Summing from $t = 2$ to $T$ proves the result.
\end{proof}

\Cref{cor:gamma_T_upper_bound} does not improve \Cref{tab:gamma_T_rates}, where the upper bounds and lower bounds of $\gamma_T$ already match. 
On the other hand, \Cref{cor:gamma_T_upper_bound} provides an alternative approach to obtain upper bounds on $\gamma_T$. 
As an example, we consider the $\gamma$-exponential kernel and $\mathcal{X} = [0,1]^d$. 
Then, the proof of \Cref{lem:remove_log_MOSS} gives $R^*_T = O(T^{(\gamma+2d)/(2\gamma + 2d)})$, where the proof is independent of $\gamma_T$. 
Observe that
\begin{align}
    \sum_{t=2}^T t^{2(\gamma+2d)/(2\gamma + 2d)}/t^2
    = \sum_{t=2}^T t^{-\gamma/(\gamma+d)}
    \leq \int_{1}^T t^{-\gamma/(\gamma+d)} \ dt
    = \frac{\gamma + d}{d}  (T^{d/(\gamma+d)} - 1).
\end{align}
Define $U_t := C t^{(\gamma+2d)/(2\gamma + 2d)} $, where $C > 0$ is some constant chosen large enough such that $R^*_t \leq U_t$ for all $t \in \mathbb{Z}_{+}$.
Applying \Cref{cor:gamma_T_upper_bound} with the sequence $(U_t)_{t\in\mathbb{Z}_{+}}$ gives 
$\gamma_T = O(T^{d/(\gamma + d)})$, which matches the $\Theta(T^{d/(\gamma+d)})$ rate of $\gamma_T$ in \Cref{tab:gamma_T_rates}.
More generally, \Cref{cor:gamma_T_upper_bound} could be useful for kernels where the upper bounds of $\gamma_T$ are not established or hard to derive, but upper bounds on $R^*_T$ are available.
On the other hand, the immediate upper bound $\gamma_T = O(U_T^2 \log (T) / T)$ only gives $\gamma_T = O(T^{d/(\gamma+d)}\log T)$, which has an extra $\log T$ factor.

\end{document}